\RequirePackage{rotating}
\documentclass[smallcondensed]{svjour3}     

\smartqed  
\usepackage{verbatim,amscd,graphicx,booktabs,multirow,graphics,xcolor,url}
\RequirePackage{rotating}
\usepackage{adjustbox}
\usepackage{footnote}
\usepackage{amsmath}
\usepackage{slashbox}
\usepackage{amssymb,amsfonts}
\let\proof\relax

\usepackage{amsthm}
\usepackage{lipsum}
\usepackage{amsmath}
  \usepackage{paralist}

  \usepackage{graphics} 
  \usepackage{epsfig} 
\usepackage{graphicx}  \usepackage{epstopdf}
 \usepackage[colorlinks=true]{hyperref}
\hypersetup{urlcolor=blue, citecolor=red}

\def\currentvolume{X}
 \def\currentissue{X}
  \def\currentyear{200X}
   \def\currentmonth{XX}
    \def\ppages{X--XX}

\DeclareMathOperator{\diag}{diag}	

\usepackage{amsthm,amssymb,amsfonts}
\usepackage{amsopn,amstext}
\usepackage{bm}
\usepackage[graphicx]{realboxes}
\usepackage{float}
\usepackage{subcaption}
\usepackage{mathrsfs}
\usepackage{placeins}
\usepackage{chemformula}
\usepackage{listing}
\usepackage{xcolor}
\usepackage{multirow}
\usepackage{mathtools, eucal}
\usepackage[noend]{algpseudocode}
\usepackage[justification=centering]{caption}
\usepackage[bottom]{footmisc}
\usepackage{nccmath}
\usepackage{arydshln}
\usepackage{comment}
\usepackage{pdflscape}
\usepackage{slashbox}
\usepackage{color}

\newcommand{\lp}{\left(}
\newcommand{\rp}{\right)}
\newcommand{\R}{\mathbb{R}}

\newcommand{\ve}{\mbox{vec}}
\newcommand{\hv}{\mbox{hvec}}
\newcommand{\nct}{\frac{n(n+1)}{2}}
\newcommand{\RNum}[1]{\uppercase\expandafter{\romannumeral #1\relax}}

\begin{document}
\title{Sparse Universum Quadratic Surface Support Vector Machine Models for Binary Classification 
}

\titlerunning{ Sparse Universum Quadratic Surface Support Vector Machine 

}        

\author{	Hossein Moosaei \and
	Ahmad Mousavi \and\\
		Milan Hlad\'{i}k \and
 Zheming Gao$^*$ \and
}
	


\institute{$*$ Corresponding Author \at
	\and
	Hossein Moosaei \at
	Department of Mathematics, Faculty of Science, University of Bojnord, Bojnord, Iran\\
	Department of Applied Mathematics, School of Computer Science, Faculty of Mathematics and Physics, Charles University, Prague, Czech Republic\\
	\email{hmoosaei@gmail.com, moosaei@ub.ac.ir, hmoosaei@kam.mff.cuni.cz }
		\and
		Ahmad Mousavi\at
	Institute for Mathematics and its Applications, University of Minnesota, Minneapolis, MN 55455, USA \\
	\email{amousavi@umn.edu }
	\and 
		 Milan Hlad\'{i}k  \at
	Department of Applied Mathematics, Faculty  of  Mathematics  and  Physics, Charles University, Prague, Czech Republic \\
	\email{hladik@kam.mff.cuni.cz}     
	\and
	Zheming Gao \at
	College of Information Science and Engineering, Northeastern University, Shenyang, Liaoning 110819, China\\
	\email{tonygaobasketball@hotmail.com}
}
\maketitle

\begin{abstract}
In binary classification, kernel-free linear or quadratic support vector machines are proposed to avoid dealing with difficulties such as finding appropriate kernel functions or tuning their hyper-parameters. Furthermore, Universum data points, which do not belong to any class, can be exploited to embed prior knowledge into the corresponding models so that the generalization performance is improved. In this paper, we design novel kernel-free Universum quadratic surface support vector machine models. Further, we propose the L1 norm regularized version that is beneficial for detecting potential sparsity patterns in the Hessian of the quadratic surface and reducing to the standard linear models if the data points are (almost) linearly separable. The proposed models are convex such that standard numerical solvers can be utilized for solving them. Nonetheless, we formulate a least squares version of the L1 norm regularized model and next, design an effective tailored algorithm that only requires solving one linear system. Several theoretical properties of these models are then reported/proved as well. We finally conduct numerical experiments on both artificial and public benchmark data sets to demonstrate the feasibility and effectiveness of the proposed models. 
\end{abstract}

\section{Introduction} \label{sec: introduction}
Machine learning has been able to dramatically thrive in many aspects of our lives such as heart disease diagnostic, text categorization, computational biology, bioinformatics, image classification, lung cancer, colon tumor diagnostic, agriculture, prediction of cryptocurrency returns, electric load forecasting, etc. \cite{akyildirim2021prediction,arabasadi2017computer,bazikar2020dc,cai2004application,javadi2019learning,ketabchi2019improvement,luo2018benchmarking,mohammadi2020statistical,mohammadi2021finite,mohammadi2021ultrasound,mucherino2009survey,noble2004support,rezaee2020discrete,rezaee2020event,wang2011color}. Therefore, it naturally has absorbed much attention from a wide range of research communities. Nevertheless, there is a tremendous number of mathematical problems raised by machine learning that still need to be studied. As an inevitable task in machine learning, the binary classification plays a key role.

For the first time, Vapnik et al. presented the support vector machine (SVM) model for the binary classification problem \cite{vapnik1974theory}. By solving a convex quadratic programming problem, the SVM model finds two unique parallel supporting hyperplanes that obtain the maximum margin. Later, Weston et al. introduced a new algorithm, called the Universum support vector machine ($\mathfrak U$-SVM), that leverages Universum data by minimizing the number of observed contradictions, which is an alternative concept to maximizing the minimum margin approach \cite{vapnik1998statistical,weston2006inference}. The $\mathfrak U$-SVM allows one to encode prior knowledge by incorporating Universum samples that do not belong to any of the classes. This model has experimentally demonstrated that it delivers a better accuracy compared to the methods that only make use of labeled data. Further, many Universum models have been recently proposed that improve their parental models by increasing the classification accuracy \cite{qi2012twin,xiao2020new,sinz2007analysis}. 

However, most real-world applications are not linearly separable, which necessitates developing robust techniques that can deal with nonlinear situations properly. The most well-established technique to manage nonlinearly separable data sets hopes in the existence of a nonlinear mapping that takes the original data to a higher-dimensional (possibly even infinite-dimensional) feature space where the transferred data is linearly separable. Perhaps the main factor of applicability of the kernel technique comes from the fact that this mapping is not required to be known at all, and, basically, this method reduces to play with appropriate kernel functions that own helpful structures. Nonetheless, it is not generally clear how to choose an appropriate kernel function, and further, tuning the involved hyper-parameters consumes plenty of computational efforts. Consequently, it is natural to think of practical methods that seek for nonlinear classifiers in the original space.

The last decade has observed many rigorous kernel-free models that can directly handle nonlinearly separable data sets without mapping them to a larger feature space. The quadratic surface support vector machine (QSSVM) \cite{dagher2008quadratic,luo2016soft} utilizes a quadratic surface directly for separating the two classes of data. Bai et al.~\cite{Bai2015} proposed a kernel-free least squares QSSVM for disease diagnosis. Gao et al.~\cite{gao2019quadratic} proposed a least squares twin QSSVM by capturing the data with two quadratic surfaces. In addition, a kernel-free double-well potential support vector machine proposed in \cite{gao2021kernel} aims a special fourth-order polynomial separating surface. 
Recall that a well-received tool for handling high-dimensional data is sparsity \cite{mousavi2019survey,mousavi2019solution,shen2018least,shen2019exact}. Once we seek for nonlinear classifiers, the number of involved decision variables in the corresponding optimization program increases dramatically, which can lead to computational complexities and overfitting. 
Therefore, incorporating a surrogate that promotes sparsity of decision variables seems beneficial. There could be other advantages as well. For example, a main drawback for the QSSVM is that it does not produce a hyperplane even if the data set is linearly separable; a natural expectation that one requires the QSSVM to acquire. To resolve this shortcoming, the authors in \cite{mousavi2019quadratic} proposed the L1-QSSVM that incorporates an extra $\ell_1$ norm regularizer into the objective function and proved that if the penalty parameter of the regularizer is large enough (with a finite lower bound), then this property is obtained. In addition, under similar conditions, this new model can capture the sparsity pattern of the Hessian matrix of the true quadratic surface.

In this paper, we propose a novel $\ell_1$ norm regularized kernel-free  quadratic surface support vector machine (\ref{formulation: L1-U-SQSSVM}) that accommodates not only noises and outliers  but also Universum data. Next, we derive its least squares version of (\ref{lsqReform}). This new reformulation adopts the $\ell_2$ norm instead of the $\ell_1$ norm to penalize the slack variables and also replaces the inequality constraints  with equality constraints. This smart shaving leads to a fast algorithm that enjoys a satisfactory generalization performance. To verify the effectiveness of the proposed methods, we report experiments carried out on several data sets, including artificial and public benchmark data sets. These experiments also demonstrate the effectiveness and efficiency of the proposed models compared with other well-known SVM models.

The rest of this paper is organized as follows. in Section \ref{sec: related_works}, we review some well-established kernel-free models that are the cornerstones of the proposed models in this paper. Section \ref{sec: Universum_quadratic_models} proposes the new models of the paper and  discusses the motivations behind them. A fast and tailored algorithm is designed in Section \ref{sec: least-square-algorithm}. We bring several theoretical results to support our models in Section \ref{sec: theoretical_properties}. Numerical experiments are reported in Section \ref{Sec: Numerical_Experiments}. Finally, conclusions are made in Section \ref{sec: Conclusions}.

\paragraph{Notation.}
The $ n$-dimensional real vector space is denoted by $\mathbb R^n$. For $x\in \mathbb R^n,$  $|x|$ is defined element-wise. $A^T $ and $ \|\cdot\| $ are notations for  the transpose of a matrix $ A $ and the Euclidean norm, respectively.
 Let $ f $ be a real valued function on  $ \mathbb  R^n $; its gradient at a point $ x $ is represented by the $n$-dimensional column vector $\nabla f(x)$. 
Let $S_n$ denote all the $n\times n$ symmetric matrices. For $A\in S_n$ we use $A\succ 0$ to say that $A$ is positive definite. Given two matrices $A \in \R^{m\times n}$ and $ B \in \R^{p \times q}$, their Kronecker product is defined as
$$A \otimes B \, := \, \left[
\begin{array}{ccc}
a_{11}  B & \cdots & a_{1n}  B \\
\vdots & \vdots & \vdots \\
a_{m1}  B & \cdots & a_{mn} B
\end{array}\right] \, \in \, \R^{mp \times nq}.$$

\section{Problem Statement and Related Works} \label{sec: related_works}

We start by describing the fundamental kernel-free model proposed for binary classification when dealing with a(n) (almost) linearly separable data set. 

\subsection{Support Vector Machine}
Suppose that we are faced with a binary classification problem and the related data set is described as follows:

\begin{equation}\label{eq: data_set_T}
T=\left\{ (x_1,y_1),\dots,(x_m,y_m)\right\}\in \left(\mathbb R^n \times\{\pm 1 \}\right)^m,
\end{equation}
 where  the $x_i$ are $n$-dimensional samples and the $y_i$  are their corresponding labels. For such a two-class problem when the classes are (almost) linearly separable, the SVM algorithm finds  two parallel hyperplanes with the maximum minimum margin.
The training  samples from the two classes   are divided by  the middle 
hyperplane  i.e., $f(x)=w^{T}x+b=0$, if and only if  $y_i(w^{T}x_i+b)\geq 1,\  i = 1,\ldots,m. $
 The hyperplane for the separable case can be obtained by solving the following minimization problem:
\begin{eqnarray} \nonumber\label{eqi01}
\mathop {\min }\limits_{w,b \,\,\,} \,&& \frac{1}{2}{{\left\| w \right\|}^{2}}\\\nonumber
\mbox{s.t.} && y_i\left(x^T_iw+b\right)\geq 1,\quad  i = 1,\ldots,m,\nonumber 
\\ && w\in \R^n, \, b\in \R, \nonumber
\end{eqnarray}
and for the nonseparable case, the SVM model can be described by the following  optimization problem:
\begin{eqnarray}\tag{SVM}
\mathop {\min }\limits_{ w,b,\xi} \,&& \frac{1}{2}{{\left\| w \right\|}^{2}}+\mu\sum_{i=1}^m\xi_i \nonumber\\
\mbox{s.t.} && 
y_i\left(x^T_iw+b\right)\geq 1-\xi_i, \quad i=1,\dots,m, \nonumber\\
&&  w\in \R^n, \, b\in \R, \ \xi \in \R^m_+, \nonumber
\end{eqnarray}
where $\mu> 0$  is a penalty parameter to control the trade-off and  $\xi_i$'s  are slack variables to accommodate noisy data points and outliers.

\subsection{Universum Support Vector Machine}
\label{subsection: Universum Suppport Vector Machine}

Suppose that we are faced with a binary classification problem but the training set $ \tilde{T} $  consists of two subsets, that is, 
$$ \tilde{T}= T\cup \mathfrak{U}, \qquad \mbox{where} \qquad \mathfrak{U} =\{ u_1,\dots,u_r \},$$
where $T$ is defined in (\ref{eq: data_set_T}) and  $ \mathfrak{U} \in \mathbb R^{ r \times n} $ denotes the Universum class and each element of $\mathfrak{U} $  represents a Universum sample. 
 To facilitate  reading the paper and understanding the models,  we use index $i$ solely for data points in a given class and we use index $j$  solely for Universum data. Further, for presenting the related mathematical formulations, we associate the Universum data with both classes to have 
\begin{align*}
(u_1,1),\dots,(u_r,1),(u_1,-1),\dots,(u_r,-1).
\end{align*}
More specifically, for unifying the Universum constraints with the remaining constraints, we need the following notation:
\begin{align*} 
&&u_j&=u_j;\quad y_j=1,  & j=1, \dots, r,&&\\
&&u_{r+j}&=u_j;\quad y_{r+j}=-1, &j=1, \dots, r.
\end{align*}
The Universum support vector machine ($\mathfrak{U}$-SVM) was proposed and formulated by Weston et al.~\cite{weston2006inference}  as follows.

\begin{eqnarray*}\tag{$\mathfrak U$-SVM}\label{SVMUniversum}
\mathop {\min }\limits_{w, b,\xi,\psi} && \frac{1}{2}{{\left\| w \right\|}^{2}}+\mu\sum\limits_{i=1}^{m}{ \xi _{i} }+C_u \sum\limits_{j=1}^{2r}{ \psi_{j} } \nonumber \\
\mbox{s.t.} && y_i \left( {{x}_{i}^Tw}+b \right) \geq 1 - {{\xi }_{i}}, \quad i=1,\dots,m, \\ \nonumber
&& y_j \left( u_j^Tw+b \right) \geq -\varepsilon - {{\psi_{j}}}, \quad j=1,\dots,2r, \\ \nonumber
&& 
w\in \R^n, \, b\in \R, \, \xi\in \R^m_+, \psi\in \R^{2r}_+,
\end{eqnarray*}
where positive parameters $\mu$ and $C_u$ control the trade-off between the minimization of training errors and the maximization of the number of Universum samples, respectively.
Further, $\varepsilon>0$ is a parameter for the $\varepsilon$-insensitive tube.   For the case of $C_u=0,$  this formulation is reduced to a standard SVM classifier \cite{cherkassky2007learning}. In general, the parameter $\varepsilon$ can be set as a small positive value.

\subsection{Quadratic Surface Support Vector Machine}
Here, the goal is to find a quadratic surface $f(x)= \frac{1}{2} x^TWx+x^Tb+c=0$ for classifying two data sets that are (almost) quadratically separable. By approximating the margin of a quadratic surface, one can formulate the quadratic surface support vector machine model as follows \cite{dagher2008quadratic}:
\begin{eqnarray}\tag{QSSVM}  \nonumber
\mathop {\min }\limits_{W, b, c} \,&& \sum_{i=1}^m \|  W x_i +    b\|_2^2  \nonumber\\
\mbox{s.t.} && 
y_i\lp  \frac{1}{2} x_i^T   W x_i + x_i^T    b + c\rp  \geqslant 1, \quad i=1,\dots,m, \nonumber\\
&&    W\in  S_n, \,    b \in \R^n, \, c \in \R  \nonumber.
\end{eqnarray}
To compromise with  possible noise and outliers in the data,  the following  soft margin version of QSSVM penalizes mis-classifications \cite{luo2016soft}:
\begin{eqnarray}\tag{SQSSVM}\label{formulation: SQSSVM}
\mathop {\min }\limits_{ W,b, c}\,&& \sum_{i=1}^m \|  W x_i +    b\|_2^2 + \mu \sum_{i=1}^m     \xi_i \nonumber\\
\mbox{s.t.} &&     y_i \lp \frac{1}{2} x_i^T   W x_i + x_i^T    b + c \rp \geqslant 1-     \xi_i, \quad i=1,\dots,m, \nonumber\\
&&    W\in  S_n,     b \in \R^n,  c \in \R,    \xi \in \R^m_+ \nonumber.
\end{eqnarray}
 If the data is linearly separable, the above quadratic models do not necessarily reduce to the original SVMs, that is,  their optimal $W^*$ is not the zero matrix. Hence, the following model is introduced \cite{mousavi2019quadratic}:
\begin{eqnarray}\tag{L1-SQSSVM}\label{formulation: L1-SQSSVM} 
\mathop {\min }\limits_{ W,b, c,\xi} \,&& \sum_{i=1}^m \|  W x_i +    b\|_2^2 +\lambda \sum_{i\le j}\left|W_{ij}\right|+\mu \sum_{i=1}^m     \xi_i \nonumber\\
\mbox{s.t.} && 
y_i \lp \frac{1}{2} x_i^T   W x_i + x_i^T    b + c \rp \geqslant 1-     \xi_i, \quad i=1,\dots,m, \nonumber\\
&&   W\in  S_n, \,    b \in \R^n, \, c \in \R, \,    \xi \in \R^m_+. \nonumber
\end{eqnarray}

\section{Universum Quadratic Surface Support Vector Machines } \label{sec: Universum_quadratic_models}

To facilitate understanding the proposed models in this main section of the paper,  we first refer to Fig.~\ref{fig:Universum}, which depicts the geometry of the Universum models clearly. 
 \begin{figure*}[htbp]
	\centering
  	\includegraphics[width=0.38\textwidth]{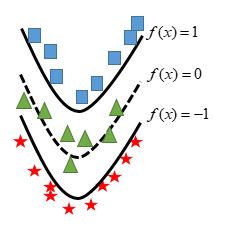}
  	\caption{ Universum QSSVM (Universum data are represented by green triangles).}
  	\label{fig:Universum}
  		\end{figure*}
\newline
We denote the hinge loss function by ${H_{ - \varepsilon }}\left[ t \right]$ given as:
  	\begin{align*}
  	{H_{ - \varepsilon }}\left[ t \right] = \max \left\{ {\begin{array}{*{20}{c}}
  		{0,}&{ - \varepsilon  - t}
  		\end{array}} \right\} = \left\{ {\begin{array}{*{20}{c}}
  		0& \quad {t >  - \varepsilon, }\\
  		{ - \varepsilon  - t}& \qquad \mbox{otherwise.} 
  		\end{array}} \right.
  	\end{align*}
  The $\varepsilon-$insensitive loss function can be defined as:
  	\begin{align*}
  	{\rho}[t] = {H_{ - \varepsilon}}\left[ t \right] + {H_{ - \varepsilon }}\left[ { - t} \right].
  	\end{align*}
  	Fig.~\ref{fig:Loss}$(a)$,  Fig.~\ref{fig:Loss}$(b)$, and Fig.~\ref{fig:Loss}$(c)$, graphically show the hinge, $\varepsilon-$insensitive, and quadratic loss functions,  respectively.
  		\begin{figure*}[!ht]
  		\centering
  	\includegraphics[width=0.8\textwidth]{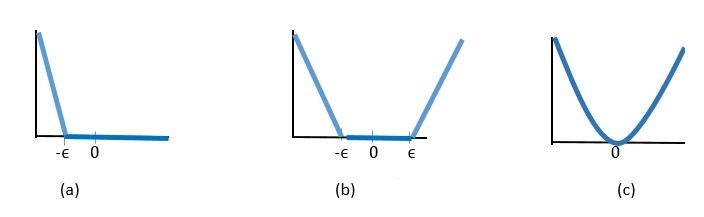}
  	\caption{{\hm ~$(a)$ hinge loss function; $ (b)$ $ \varepsilon-$insensitive loss function; $ (c)$ quadratic loss function.}}
  	\label{fig:Loss}       
  		\end{figure*}
Thus, since we assume that  Universum data falls in an $\varepsilon$-tube neighborhood  of the classifier, we can embed the prior knowledge in  Universum data by using the parametric  $ \varepsilon-$insensitive loss function.	The geometric interpretation of this loss function for Universum data is described in Fig \ref{fig:newloss1}.
  	\begin{figure*}[!ht]
  	\centering
  	\includegraphics[width=0.7\textwidth]{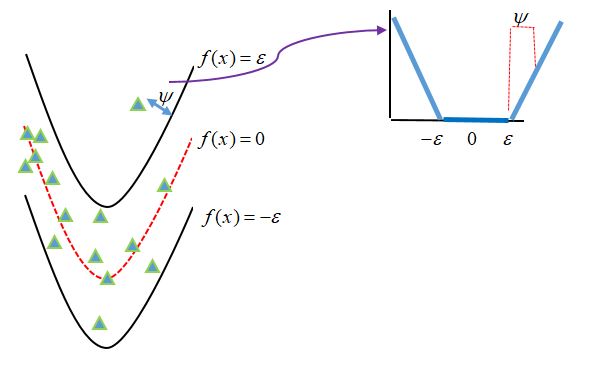}
  	\caption{  $ \varepsilon-$insensitive loss function for Universum data set.}
  	\label{fig:newloss1}       
  		\end{figure*}
The prior knowledge embedded in Universum data can be reflected in the summation of parametric  $ \varepsilon-$insensitive loss functions, i.e., $ \sum_{j = 1}^{2r} {{\rho}[f(u_j)]}$. It is clear that the smaller value of this summation leads to higher prior possibility of correct classification, and vice versa. Thus, by adding ${C_u}\sum_{j = 1}^{2r} {{\rho}[f(u_j)]}$ in the objective function of QSSVM, we introduce our new model, the Universum-quadratic surface  support vector machine, as follows.

\begin{eqnarray}\tag{$\mathfrak U$-QSSVM}\label{formulation: U-QSSVM}
\mathop {\min }\limits_{  W,    b, c} \,&& \sum_{i=1}^m \|  W x_i +    b\|_2^2 \nonumber\\
\mbox{s.t.} && y_i \lp \frac{1}{2} x_i^T   W x_i + x_i^T    b + c \rp \geqslant 1, \quad  i=1,\dots,m, \nonumber\\
&&  y_j \lp \frac{1}{2} u_j^T   W u_j + u_j^T    b + c \rp \geqslant -\varepsilon, \quad  j=1,\dots,2r, \nonumber\\
           &&   W\in  S_n,\   b \in \R^n,\ c \in \R. \nonumber
\end{eqnarray}

To account for possible noise and outliers in the data,  the  soft margin loss function is used. We suppose the classifier has the form
$ f_{W,b,c}(x)=\frac{1}{2} x^T   W x + x^T    b + c$, so the objective function of $\mathfrak U$-QSSVM can formulated as:
	\begin{align}\label{eq0l1}
  \sum_{i=1}^m	\|  W x_i +    b\|_2^2 + \mu \sum_{i=1}^m     {{H_1}[f_{W,b,c}(y_ix_i)]}+{C_u}\sum\limits_{j = 1}^{2r} {{\rho}[f_{W,b,c}(u_j)]} ,
  	\end{align}
where $\mu$, $ {C_u} $ controls the loss of samples and impact of  Universum data, respectively. Hence, the soft margin version that also include possible noisy Universum data  can be formulated as follows:
\begin{eqnarray}\tag{$\mathfrak U$-SQSSVM}\label{formulation: U-SQSSVM}
\mathop {\min }\limits_{W,b, c,\xi, \psi} \,&& \sum_{i=1}^m \|  W x_i +    b\|_2^2 + \mu \sum_{i=1}^m     \xi_i + C_u \sum_{j=1}^{2r}     \psi_j  \nonumber\\
\mbox{s.t.} && y_i \lp \frac{1}{2} x_i^T   W x_i + x_i^T    b + c \rp \geqslant 1-     \xi_i, \quad i=1,\dots,m,  \nonumber\\
&& y_j\lp \frac{1}{2} u_j^T   W u_j + u_j^T    b + c \rp \geqslant -\varepsilon-\psi_{j},\quad  j=1,\dots,2r,  \nonumber\\
&&    W\in  S_n, \,    b \in \R^n, \, c \in \R, \,    \xi \in \R^m_+,\,    \psi \in \R^{2r}_+. \nonumber
\end{eqnarray}
In \cite{mousavi2019quadratic}, authors demonstrated that introducing  an $\ell_1$ norm regularization term (penalizing the Hessian matrix components of the classifier) in the objective function benefits the  quadratic surface SVM model in terms of reducing to the original SVM if the data points are linearly separable and further, capturing sparsity of the true Hessian matrix when the corresponding  penalty parameter ($\lambda$ below) is large enough. Thus, we also propose the $\ell_1$ norm regularized version of the above model below:
\begin{eqnarray}\label{formulation: L1-U-SQSSVM} \nonumber \\
\mathop {\min }\limits_{  W,    b, c,    \xi,\psi} \,&& \sum_{i=1}^m \|  W x_i +    b\|_2^2 + \lambda \sum_{i\le j}\left|W_{ij}\right| + \mu \sum_{i=1}^m     \xi_i + C_u \sum_{j=1}^{2r}     \psi_j  \nonumber\\
\mbox{s.t.} &&  y_i \lp \frac{1}{2} x_i^T   W x_i + x_i^T    b + c \rp \geqslant 1-     \xi_i, \quad i=1,\dots,m,  \nonumber\\
 && y_j\lp \frac{1}{2} u_j^T   W u_j + u_j^T    b + c \rp \geqslant -\varepsilon-\psi_{j},\quad  j=1,\dots,2r,  \nonumber\\
&&   W\in S_n, \,    b \in \R^n, \, c \in \R, \,    \xi \in \R^m_+, \,   \psi \in \R^{2r}_+. \tag{L1-$\mathfrak U$-SQSSVM}
\end{eqnarray}
The quadratic surface based models mentioned or introduced above are not in the standard forms of quadratic programs. So, We now introduce some notations and give several definitions that will help us achieve this goal. We shall be succinct here as these definitions are mainly borrowed from \cite{mousavi2019quadratic}. 
For a square matrix $  A = [a_{ij}]_{i=1,\dots,n;j=1,\dots,n}\in \mathbb R^{n\times n}$, its vectorization given by
$$\ve(A) := \left[a_{11},\dots, a_{n1},a_{12},\dots,a_{n2},\dots,a_{1n},\dots,a_{nn}\right]^T \in \R^{n^2}.$$
In case $A$ is symmetric, $\ve(A)$ contains redundant information so that we often consider its half-vectorization  given by:
$$\hv(  A) := \left[a_{11},\dots, a_{n1},a_{22},\dots,a_{n2},\dots,a_{nn}\right]^T \in \R^{\nct}.$$
Given $n\in \mathbb N$,  there exist a unique elimination matrix $  L_n \in \R^{\nct \times n^2}$ such that \cite{magnus1980elimination}
$$  L_n \ve(  A) = \hv(  A);\quad \forall \,   A \in S_n,$$ 
and further, this elimination matrix $  L_n$ has full row rank \cite{magnus1980elimination}. Conversely, for any $n\in \mathbb N$, there is a unique duplication matrix $  D_n \in \R^{n^2 \times \nct}$ such that 
$$  D_n \hv(  A) = \ve(  A); \quad \forall\,   A \in  S_n \qquad \mbox{and} \qquad L_nD_n=I_\nct.$$ 

\begin{definition} \label{def: definitios}
Let 
\begin{eqnarray}\nonumber
 s_i := &&  \frac{1}{2}\hv( x_i x_i^T), \qquad \forall i = 1, \dots, m \\ \nonumber
s_j := &&  \frac{1}{2}\hv(u_ju_j^T), \qquad \forall  j = 1, \dots,2r \\ \nonumber
r_i := && [s_i;x_i], \qquad \forall  i = 1, \dots,m \\ \nonumber
r_j:= && [s_j;u_j], \qquad \forall  j = 1, \dots, 2r \\ \nonumber
w := && \hv(W), \\ \nonumber
z: = && [w;b],\\ \nonumber
V:= && \begin{bmatrix} I_{\frac{n(n+1)}{2}} &   0_{\frac{n(n+1)}{2}\times n} \end{bmatrix}, \qquad \forall  i = 1,\dots,m \\ \nonumber
X_i : =&& I_n \otimes x_i^T, \qquad \forall  i = 1, \dots,m \\ \nonumber
M_i := && X_i D_n, \qquad \forall  i = 1, \dots, m \\ \nonumber
H_i:= && 
\begin{bmatrix} M_i &   I_{ n}\end{bmatrix},\\ \nonumber
 G := && 2 \sum_{i=1}^{m} H_i^T H_i, \\ \nonumber
X:=&& [x_1^T;x_2^T,\dots;x_m^T;u_1^T;u^T_2;\dots;u_r^T].
\end{eqnarray}
\end{definition}
Consequently, for fixed $i = 1, \dots, m $ and $j = 1, \dots, 2r $, we get the following equations:
\begin{equation*}
     \left\{ \begin{array}{ll}
\frac{1}{2 }x_i^TW x_i + x_i^Tb + c=  z^T r_i+ c, & \\
\frac{1}{2} u_j^TW u_j + u_j^Tb + c=  z^T r_j+ c, & \\
Wx_i  =  X_i\ve( W) \, = X_i   D_n \hv(  W) \, = \, M_i \hv(  W) \, = \, M_iw, & \\
W x_i +    b  =  \, M_iw +  I_nb \, = \, H_iz, & \\
\sum_{i=1}^m \|Wx_i+b\|_2^2 =\sum_{i=1}^m \left(H_i    z\right)^T\left(H_i z\right) \,= \,    z^T \left[\sum_{i=1}^{m} \left(H_i\right)^T H_i\right]    z=\frac{1}{2}z^TGz.
\end{array} \right.
\end{equation*}
The above definitions and equations lead to the following standard quadratic program:
\begin{eqnarray}\tag{$\mathfrak U$-SQSSVM'}\label{formulation: U-SQSSVM'} \nonumber
\min_{z, c, \xi,\psi} \,&&  \frac{1}{2}z^TGz+ \mu \sum_{i=1}^m     \xi_i + C_u \sum_{j=1}^{2r}     \psi_j \nonumber\\
\mbox{s.t.} && y_i \lp z^Tr_i +  c \rp \geqslant 1-     \xi_i, \quad i=1,\dots,m, \nonumber\\
&&  y_j\lp z^Tr_j+c \rp \geqslant -\varepsilon-\psi_{j},\quad  j=1,\dots,2r, \nonumber\\
&&    z \in \R^{\frac{n(n+1)}{2}+n}, \, c \in \R, \,    \xi \in \R^m_+,\,     \psi \in \R^{2r}_+, \nonumber
\end{eqnarray}
and the following quadratic model with an $\ell_1$ norm regularizer:
\begin{eqnarray}\tag{L1-$\mathfrak U$-SQSSVM'}\label{formulation: L1-U-SQSSVM'}  \nonumber
\min_{z, c, \xi,\psi} \,&&  \frac{1}{2}z^TGz+ \lambda \|Vz\|_1+ \mu \sum_{i=1}^m     \xi_i + C_u \sum_{j=1}^{2r} \psi_j  \nonumber\\
\mbox{s.t.} && y_i \lp z^Tr_i +  c \rp \geqslant 1-     \xi_i, \quad i=1,\dots,m,  \nonumber\\
&& y_j\lp z^Tr_j+c \rp \geqslant -\varepsilon-\psi_{j},\quad  j=1,\dots,2r,  \nonumber\\
&&   z \in \R^{\frac{n(n+1)}{2}+n},  \, c \in \R, \,    \xi \in \R^m_+,\, \psi \in \R^{2r}_+.  \nonumber
\end{eqnarray}
These models are all convex such that efficient optimization solvers can be utilized for solving them. Nonetheless, we show that the least squares version of the mainly proposed model can have a fast tailored algorithm.

 \section{Least Squares Universum Quadratic Surface Support Vector Machine} \label{sec: least-square-algorithm}

In this section, we propose the least squares version of (\ref{formulation: L1-U-SQSSVM}) and introduce a fast and efficient method for solving it.
The least square model is called  LS-L1-$\mathfrak U$-SQSSVM, which seeks a  quadratic classifier like as well, but here to measure the empirical risk we use quadratic loss function instead of the other loss functions so the objective function of (\ref{eq0l1}) will be changed to:
	\begin{align}\label{eq001}
  \sum_{i=1}^m    	\|  W x_i +    b\|_2^2 + \mu \sum_{i=1}^m     {	\| 1-f_{W,b,c}(y_ix_i)\|^2}+{C_u}\sum\limits_{j = 1}^{2r} {	\| f_{W,b,c}(u_j)\|^2}.
  	\end{align}
By introducing the slacks variables $\xi$ and $\psi$, and after a reformulation, we  have the following problem:

\begin{eqnarray}\label{lsqReform} \nonumber \\
\mathop {\min }\limits_{  W,    b, c,    \xi,\psi} \,&& \sum_{i=1}^m \|  W x_i +    b\|_2^2 + \lambda \sum_{i\le j}\left|W_{ij}\right| + \mu \|\xi\|_2^2 + C_u \|\psi\|_2^2  \nonumber\\
\mbox{s.t.} &&  y_i \lp \frac{1}{2} x_i^T   W x_i + x_i^T    b + c \rp = 1-     \xi_i, \quad i=1,\dots,m,  \nonumber\\
 && y_j\lp \frac{1}{2} u_j^T   W u_j + u_j^T    b + c \rp = -\varepsilon-\psi_{j},\quad  j=1,\dots,2r,  \nonumber\\
&&   W\in S_n, \,    b \in \R^n, \, c \in \R, \,    \xi \in \R^m,    \psi \in \R^{2r}. \tag{LS-L1-$\mathfrak U$-SQSSVM}
\end{eqnarray}
In this formulation, the minimization of $\left\|\psi \right \|^{2}_2$ makes the Universum data as close as possible to hyperplane $\frac{1}{2} x_i^T   W x_i + x_i^T    b + c= \varepsilon$, the $\varepsilon$ is a very small parameter that can be even zero. The constraints of the above problem are equations, so by substituting $ \xi$ and $\psi$ into the objective function, and by virtue of Definition \ref{def: definitios},  (\ref{lsqReform}) can be reformulated as
\begin{align}\label{lsqReform2}
   \min_{z,c}\
  \dfrac{1}{2}z^{T}Gz+\lambda \|Vz \|_{1}+\mu \big \| e -D_{1}\big(A^{T}z+ce\big)\big\|^{2}_2+C_u \big \| \varepsilon e + D_{2}\big(U^{T}z+ce\big)\big \|^{2}_2, \tag{LS-L1-$\mathfrak U$-SQSSVM'}
\end{align}
where $A=[r_1, r_2,\dots,r_m]$, $U=[u_1,\dots ,u_{2r}]$ by putting $u_j=r_j$, $D_{1}=\diag (y_{i}), D_{2}=\diag (y_{j})$,  $\mu$ and $C_u$ are the penalty parameters, and $ \xi$ and $\psi$  are slack variables.

Note that we have $V = [v_1, v_2, \dots, v_{n(n+1)/2}]^T,$ 
where $v_i = [e_i^T, 0_{1\times n}]^T$ by Definition~\ref{def: definitios} and consequently, we get \begin{align}\label{4} 
(\left \| Vz\right \|_{1})'=\sum_{i=1}^{m}\dfrac{v_{i}^T z v_{i} }{|v_{i}^T z|}=\sum \dfrac{v_{i}^{T}zv_{i}}{D_{i}},
\qquad\mbox{if } \quad |v_{i}^T z|\neq 0,
\end{align}
and, by letting $D=\diag\left(\left|v_{1}^T z\right|,\left|v_{2}^T z\right|,\dots, \left|v_{\nct}^T z\right|\right)$, 
\begin{align}\label{5} 
\Big( \sum \dfrac{(v_{i}^T z)^{2}}{2D_{i}}\Big)'=\dfrac{2}{2}\sum \dfrac{v_{i}^{T}zv_{i} }{D_{i}}.
\end{align}
Since \eqref{4}=\eqref{5}, the derivative is similar and we use 
\begin{align*}
\sum \dfrac{(v_{i}^T z)^{2}}{2D_{i}}=\dfrac{1}{2} (Vz)^T D(Vz), \quad\mbox{where} \quad  D=\diag\left(\frac{1}{|Vz|}\right),
\end{align*}
instead of $\left \| Vz\right \|_{1}$. Then, (\ref{lsqReform2}) is equivalent to:
\begin{equation*}
\min_{z,c}\ \dfrac{1}{2}z^{T}Gz+\dfrac{\lambda}{2}(Vz)^{T}D(Vz)+\mu \big\| e-D_{1}\big(A^{T}z+ce\big)\big\|^{2}+C_u\big\| \varepsilon e +D_{2}\big(U^{T}z+ce\big)\big\|^{2}.
\end{equation*}

 This problem can be solved by putting the gradient with respect to $z$ and $c$ is equal to zero, so we have the following equations:
\begin{align*}
\dfrac{\partial f}{\partial z}&=Gz+\lambda V^{T}DVz-2\mu AD_{1}\big(e-D_{1}\big(A^{T}z+ce\big)\big)\\
&~~+2C_u UD_{2}\big(\varepsilon e+D_{2}
\big(U^{T}z+ce\big)\big)=0,\\
\dfrac{\partial f}{\partial c}&=-2\mu e^{T}D_{1}\big(e-D_{1}\big(A^{T}z+ce\big)\big)
+2C_u e^{T}D_{2}\big(\varepsilon e+D_{2}\big(U^{T}z+ce\big)\big)=0,
\end{align*}
 by integrating the above equations, we have the fallowing system:
\begin{align}
&\left[ \begin{matrix}
G + \lambda V^{T}DV + 2\mu AD_{1}D_{1}A^{T} +2C_u  UD_{2}D_{2}U^{T} & 2\mu AD_{1}D_{1}e + 2C_u  UD_{2}D_{2}e \\
2\mu e^{T}D_{1}D_{1}A^{T} + 2C_u  e^{T}D_{2}D_{2}U^{T} & 2\mu e^{T}D_{1}D_{1}e + 2C_u  e^{T}D_{2}D_{2}e
\end{matrix}\right] \left[ \begin{matrix}
z\\
c
\end{matrix}\right] \nonumber\\\label{eq: gradient = 0}
& = \left[ \begin{matrix}
2\mu AD_{1}e-2C_u UD_{2}\varepsilon e\\
2\mu e^{T}D_{1}e-2C_u e^{T}D_{2}\varepsilon e
\end{matrix}\right].
\end{align}
Since $D_{1}D_{1}$ and $D_{2}D_{2}$ are both identity matrices, we can solve out $z$ and $c$ as the following:

Let 
$$\Sigma = \left[ \begin{matrix}
G + \lambda V^{T}DV + 2\mu AA^{T} +2C_u  UU^{T} & 2\mu Ae + 2C_u  Ue \\
2\mu e^{T}A^{T} + 2C_u  e^{T}U^{T} & 2\mu m + 4C_u  r
\end{matrix}\right],$$
and assume it is invertible. Let 
$$\beta =\left[ \begin{matrix}
2\mu AD_{1}e-2C_u UD_{2}\varepsilon e\\
2\mu e^{T}D_{1}e-2C_u e^{T}D_{2}\varepsilon e
\end{matrix}\right],$$
then 
\begin{align*}
\left[ \begin{matrix}
z\\
c
\end{matrix}\right]= \Sigma^{-1}\beta.
\end{align*}
 The Algorithm  1 describes our proposed method.
\begin{center}
\begin{tabular}{l}
\hline Algorithm 1\\ 
\hline Input: matrices $G, A, U, D_1, D_2$, parameters $C, C_u , \varepsilon$. \\ 
Output: $z, c$\\
Initialize $z^0$, and $D^0\diag \Big(\dfrac{1}{|Vz^0|}\Big)\neq 0$. Calculate $\Sigma^0$ based on $D^0$;\\
\quad While ~~ $\left \| w^{k+1}-w^{k}\right \| >10^{-6}$\\
\quad $w^{(k+1)}=\left(\begin{matrix}
 z^{k+1}\\
c^{k+1}
\end{matrix} \right) =\Big(\Sigma^k\Big)^{-1} \beta$\\
\quad Update $D^{k+1}=\diag\left(\dfrac{1}{|Vz^{k+1}|}\right)\neq 0$, and $\Sigma^{k+1}$\\
\quad $k=k+1$\\
end. \\
\hline 
\end{tabular} 
\end{center} 
 \begin{remark}
In our computations, in the case $D^{k+1}=\diag\left(\frac{1}{\left|Vz^{k+1}\right|}\right)=0$, we approximated it by $D^{k+1}=\diag\left(\frac{1}{\left|Vz^{k+1}\right|+ \delta}\right)$, where $\delta$ is a very small positive number.
\end{remark}

In the next section, we study several theoretical properties of the models proposed above.

\section{Theoretical Properties of Proposed Models} \label{sec: theoretical_properties}

We first bring several results which their proofs have similar ideas as those cited.

\begin{theorem}[solution existence] \label{theorem:existence}
   Given any data set, the models (\ref{formulation: U-SQSSVM}) and (\ref{formulation: L1-U-SQSSVM}) obtain their optimal solutions with finite objective values.
\end{theorem}
\begin{proof}
See \cite[Theorem 4.1]{mousavi2019quadratic}.
\end{proof}
By Theorem 4.3 in \cite{mousavi2019quadratic}, the positive semidefinite  matrix $G$ introduced in Definition~\ref{def: definitios} is positive definite if and only if the sample data matrix $X$ (also from Definition~\ref{def: definitios}) has linearly independent columns and the vector of ones $\textbf{e}$ is not in its range.  Roughly speaking, Theorem 4.4 in \cite{mousavi2019quadratic}  demonstrates that the matrix $G$ is positive definite for almost  any given data set. 

\begin{theorem}[$ z$-uniqueness]  \label{theorem:z_uniqueness}
If $G\succ 0$, 
then the solution $ z^*$ is unique in the models (\ref{formulation: U-SQSSVM'}) and (\ref{formulation: L1-U-SQSSVM'}).
\end{theorem}
\begin{proof}
See \cite[Theorem 4.2]{mousavi2019quadratic}.
\end{proof}

\begin{theorem}[vanishing margin $\xi$]  \label{theorem:xi_0_mu_large}
Assume the data set $T$ {\color{purple} defined in Section \ref{subsection: Universum Suppport Vector Machine}} is quadratically separable and $G\succ 0$. For any $\lambda$, there exists a corresponding $\underline{\mu}$ (depending on $\lambda$), such that for all $\mu>\underline{\mu}$, (\ref{formulation: L1-U-SQSSVM}) has the unique solution  $( z^*, c^*,  \xi^*=0,\psi^*)$.
\end{theorem}
\begin{proof}
See \cite[Theorem 4.5]{mousavi2019quadratic}.
\end{proof}

\begin{theorem}[upper and lower bound for $c$]\label{theorem: boundc}
Suppose that  $G\succ 0$ and let $( z^*, c^*,  \xi^*,\psi^*)$ be the optimal solution of problem (\ref{formulation: L1-U-SQSSVM'}). Then, $\underline{c} \leqslant c^* \leqslant \overline{c}$, where
\begin{align*}
\underline{c}&=\max \{\underline{\alpha},\underline{\beta}\},\\
\overline{c}&=\min \{\overline{\alpha}, \overline{\beta} \},
\end{align*}
and
\begin{align*}
\underline{\alpha}&= 
 \max_{i:1\leq i\leq m,\,y_i=1} \{1-\xi_{i}^{*}-(z^*)^T r_{i}\},\\
\underline{\beta}&= 
 \max_{j:1\leq i\leq r,\,y_i=1}    
   \{-\varepsilon-\psi_j^*-(z^*)^T r_i\},\\
\overline{\alpha}&= 
 \max_{i:1\leq i\leq m,\,y_i=-1} \{\xi_i^*-1-(z^*)^T r_i\},\\
\overline{\beta}&= 
 \max_{j:r+1\leq j\leq 2r,\,y_i=-1}    
   \{\varepsilon + \psi_{j}^{*}-(z^*)^T r_{j}\}.
   \end{align*}
\end{theorem}
\begin{proof}
The optimal solution  of  problem (\ref{formulation: L1-U-SQSSVM'}) satisfies the constraints
\begin{align}
&y_{i}\left((z^*)^T r_{i}+c^{*}\right)\geq 1-\xi_{i}^{*}, \quad \forall i=1,2,\dots,m \label{1} \\
&y_{j}\left((z^*)^T r_{j}+c^{*}\right)\geq -\varepsilon  -\psi_{j}^{*}, \quad \forall i=1,2,\dots,2r.\label{2} \end{align}
For $y_{i}=1$, the inequality \eqref{1} transfers to
\begin{align*}
(z^*)^T r_{i}+c^{*}\geq 1-\xi_{i}^{*}
\ \Rightarrow\  
c^{*}\geq 1-\xi_{i}^{*}-(z^*)^T r_{i}
\ \Rightarrow\  
c^{*}\geq \underline{\alpha}.
\end{align*}
For $y_{i}=-1$, from \eqref{1}, we have:
\begin{align*}
(z^*)^T r_{i}+c^{*}\leq \xi_{i}^{*}-1
\ \Rightarrow\  
c^{*}\leq \overline{\alpha}.
\end{align*}
Similarly, for $y_{j}=1~ (j=1,\dots, r)$ and $y_{j}=-1~ ( j=r+1,\cdots, 2r)$ in \eqref{2} we have:
\begin{align*}
(y_{j}=1)
\ \Rightarrow\  
(z^*)^T r_{j}+c^{*}\geq -\varepsilon -\psi_{j}^{*}
\ \Rightarrow\  
c^{*}\geq \underline{\beta},
\end{align*}
and for $y_{j}=-1$
\begin{align*}
(z^*)^T r_{j}+c^{*} \leq \varepsilon + \psi^{*}_{j}
\ \Rightarrow\  
c^{*} \leq \overline{\beta},
\end{align*}
which completes the proof.
\end{proof}

Naturally, the situation in which $z^*=0$ is not desirable. Thus, we focus now on conditions ensuring that the optimal solution is nonzero.

\begin{proposition}
Let $(z^*, c^*)$ and $f^*$ be an optimal solution and the optimal value of the problem ~\eqref{lsqReform2}, respectively. Denote
\begin{subequations}\label{lsqCondNonzeroZvarCXP}
\begin{align}
\tilde{c} &= \frac{\mu e^TD_1e-C_u\varepsilon e^TD_2e}{\mu m+2r C_u},\\
\tilde{\xi} &= e-D_1e\tilde{c},\\
\tilde{\psi} &= -\varepsilon e-D_2e\tilde{c}.
\end{align}
\end{subequations}
If 
\begin{align*}
f^* < \mu\|\tilde{\xi}\|^2_2+C_u \| \tilde{\psi} \|^2_2,
\end{align*}
then $z^*\not=0$.
\end{proposition}

\begin{proof}
Suppose that $z^*=0$. Then $\xi=e-D_1ec$ and $\psi=-\varepsilon e-D_2ec$. Thus $c^*$ minimizes the quadratic function
\begin{align*}
\mu\|\xi\|^2_2+C_u \| \psi \|^2_2
=c^2(\mu m+2r C_u)
 +2c(-\mu e^TD_1e+C_u\varepsilon e^TD_2e )
 +(m+2r\varepsilon^2).
\end{align*}
Its minimum has the value $c^*=\tilde{c}$, whence 
the optimal value is
\begin{align*}
f^* = \mu\|\tilde{\xi}\|^2_2+C_u \| \tilde{\psi} \|^2_2,
\end{align*}
which contradicts our assumption.
\end{proof}

Naturally, when $\mu$ and $C_u $ are small, then the corresponding terms in the problem \eqref{lsqReform2} force the optimal solution $z^*$ to be zero. Therefore, we can condition $z^*\not=0$ by taking $\mu$ and $C_u$ large enough. But which values are really large enough? For the sake of simplicity of exposition, we assume in the following that $C_u=\mu$.

\begin{proposition}
Let $\hat{z}\in \mathbb R^m$, let $(\tilde{c},\tilde{\xi},\tilde{\psi})$ from \eqref{lsqCondNonzeroZvarCXP} and denote 
\begin{align*}
\hat{c} &= \frac{ e^TD_1e-e^TA^T\hat{z}-\varepsilon e^TD_2e -e^TU^T\hat{z}}{m+2r},\\
\hat{\xi} &= e-D_1A^T\hat{z}-D_1e\hat{c},\\
\hat{\psi} &= -\varepsilon e-D_2U^T\hat{z}-D_2e\hat{c}.
\end{align*}
If 
\begin{align}\label{lsqCondNonzeroMu}
\mu > \frac{\frac{1}{2}\hat{z}^TG\hat{z}+\lambda \left \|V\hat{z} \right\|_{1}}{\|\tilde{\xi}\|^2_2+\| \tilde{\psi} \|^2_2-\|\hat{\xi}\|^2_2\|-\hat{\psi} \|^2_2}
\end{align}
and the denominator of the right-hand side is positive, then  $z^*\not=0$.
\end{proposition}

\begin{proof}
The objective value of the feasible solution $(z=0,\tilde{c},\tilde{\xi},\tilde{\psi})$ is
$$
\tilde{f}=\mu\|\tilde{\xi}\|^2_2+\mu\| \tilde{\psi} \|^2_2.
$$

Now, we show that for $\hat{z}$ the optimal values of $c$, $\xi$ and $\psi$ are $\hat{c}$, $\hat{\xi}$ and $\hat{\psi}$. From the constraints, we have $\hat{\xi}=e-D_1A^T\hat{z}-D_1e\hat{c}$ and $\hat{\psi}=-\varepsilon e-D_2U^T\hat{z}-D_2e\hat{c}$. Thus $\hat{c}$ minimizes the quadratic function (for certain $\alpha\in R$)
\begin{align*}
\|\hat{\xi}\|^2_2+ \| \hat{\psi} \|^2_2
=c^2(m+2r)
 +2c\left(- e^TD_1e+e^TA^T\hat{z}+\varepsilon e^TD_2e +e^TU^T\hat{z}\right)
 +\alpha.
\end{align*}
Its minimum is attained at the value of $c=\hat{c}$. 

The objective value of the feasible solution $(\hat{z},\hat{c},\hat{\xi},\hat{\psi})$ is
$$
\hat{f}=\frac{1}{2}\hat{z}^TG\hat{z}+\lambda \left \|V\hat{z} \right\|_{1}+\mu\|\hat{\xi}\|^2_2 + \mu\|\hat{\psi} \|^2_2.
$$
Hence we have $z^*\not=0$ provided $\tilde{f}>\hat{f}$, from which we express $\mu$ as in~\eqref{lsqCondNonzeroMu}.
\end{proof}

The question is how to find an appropriate $\hat{z}$ for which 
\begin{align*}
\|\tilde{\xi}\|^2_2+\| \tilde{\psi}\|^2_2
> \|\hat{\xi}\|^2_2\| +\|\hat{\psi} \|^2_2.
\end{align*}
If we have no candidate, we can find it in essence by solving an auxiliary optimization problem \begin{align*}
 \min_{z,c}\ 
 \big \| e -D_{1}\big(A^{T}z+ce\big)\big\|^{2}_2
 +\big \| \varepsilon e + D_{2}\big(U^{T}z+ce\big)\big \|^{2}_2.
\end{align*}
Basically, one need not to solve it to optimality. It is sufficient to obtain an objective value less that $\|\tilde{\xi}\|^2_2+\| \tilde{\psi}\|^2_2$.

\section{Numerical Experiments}
\label{Sec: Numerical_Experiments}

In this section, various numerical experiments are conducted to verify the proposed models for binary classification. We first introduce some settings of the experiments.

\subsection{Experiment Settings}
\label{subsection: effect of parameters and Universum data}

In the numerical experiments, the proposed $\mathfrak U$-SQSSVM, L1-$\mathfrak U$-SQSSVM, and LS-L1-$\mathfrak U$-SQSSVM models are implemented as well as some benchmark models, including the SQSSVM \cite{luo2016soft}, L1-SQSSVM \cite{mousavi2019quadratic}, linear soft-SVM (LSVM) \cite{cortes1995support}, and the Universum SVM ($\mathfrak U$-SVM) \cite{weston2006inference}. Besides, the radial basis function (RBF) kernel-based SVM (SVM-rbf) and the RBF kernel-based Universum SVM ($\mathfrak U$-SVM-rbf) are also implemented. All the computational experiments are conducted on a desktop equipped with four Intel (R) Core (TM) i3-9100 CPU @ 3.40GHz CPUs and 32GB RAM. Moreover, Mosek 9.2.20 is utilized as the optimization solver for several models. The machine learning software package Scikit-learn \cite{scikit-learn} is utilized to implement LSVM and SVM-rbf models. To avoid the dominance of input features with greater numerical values over other smaller values, the data points are normalized to $[0,1]$ for each tested data set. To assess the performance of the algorithms and their classification accuracy, the five-fold cross-validation strategy \cite{stone1974cross} is utilized in the experiments in Sections \ref{subsection: Experiments on artificial data sets} and \ref{subsection: experiments on public benchmark data sets}.

For each single experiment, the accuracy score is calculated  as the number of correct predictions divided by the total number of predictions; the result is then multiplied by 100 to obtain the percentage accuracy. For the artificial and the public benchmark data sets utilized in Section \ref{subsection: Experiments on artificial data sets} and \ref{subsection: experiments on public benchmark data sets}, we randomly select ten percent of data points from each class (i.e. 20\% in total), and use them to generate Universum data points by averaging each pair of samples from different classes.

\subsection{Effect of Parameters and Universum Data}
\label{subsection: effect of parameters and Universum data}

We first conduct some experiments to show how the parameters $C_u$ and $\epsilon$ affect the classification accuracy. The colormaps are plot for each proposed model when fixing parameters $\mu$ and $\lambda$:

\begin{itemize}
    \item Arti-Q2 data:  $\mu = 65536$, $\lambda = 4$.
    \item ND2by200 data:  $\mu = 65536$, $\lambda = 4$.
\end{itemize}

The $\mu$ and $\lambda$ are the optimal parameters for SQSSVM and L1-SQSSVM models when applied to the data sets.

\begin{figure}[H]
    \centering
    \begin{subfigure}[H]{0.3\textwidth}
        \includegraphics[width=\textwidth]{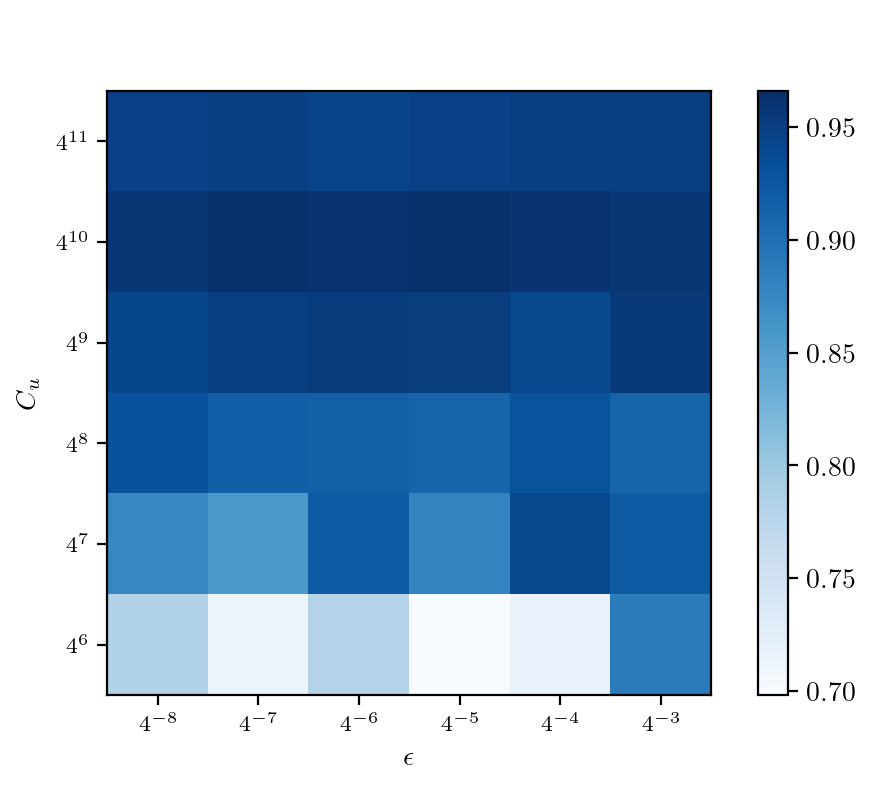}
        \caption{$\mathfrak U$-SQSSVM (Arti-Q2)}
        \label{fig: U-SQS Arti-Q2}
    \end{subfigure}
	~
    \begin{subfigure}[H]{0.3\textwidth}
        \includegraphics[width=\textwidth]{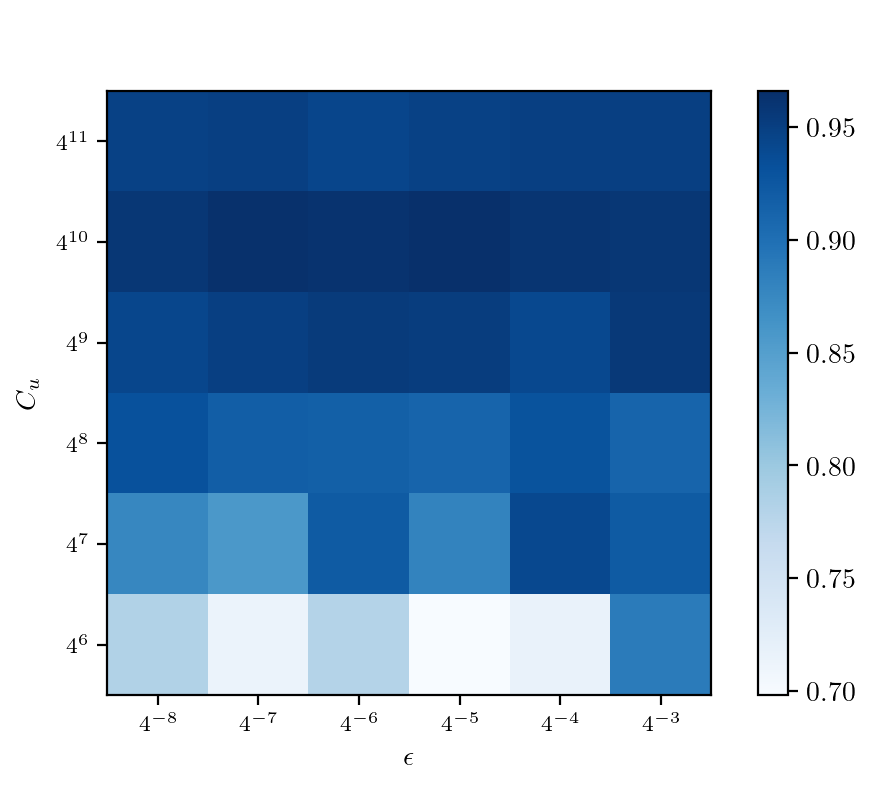}
        \caption{L1-$\mathfrak U$-SQSSVM (Arti-Q2)}
        \label{fig: L1-U-SQS Arti-Q2}
    \end{subfigure}
    ~
    \begin{subfigure}[H]{0.3\textwidth}
        \includegraphics[width=\textwidth]{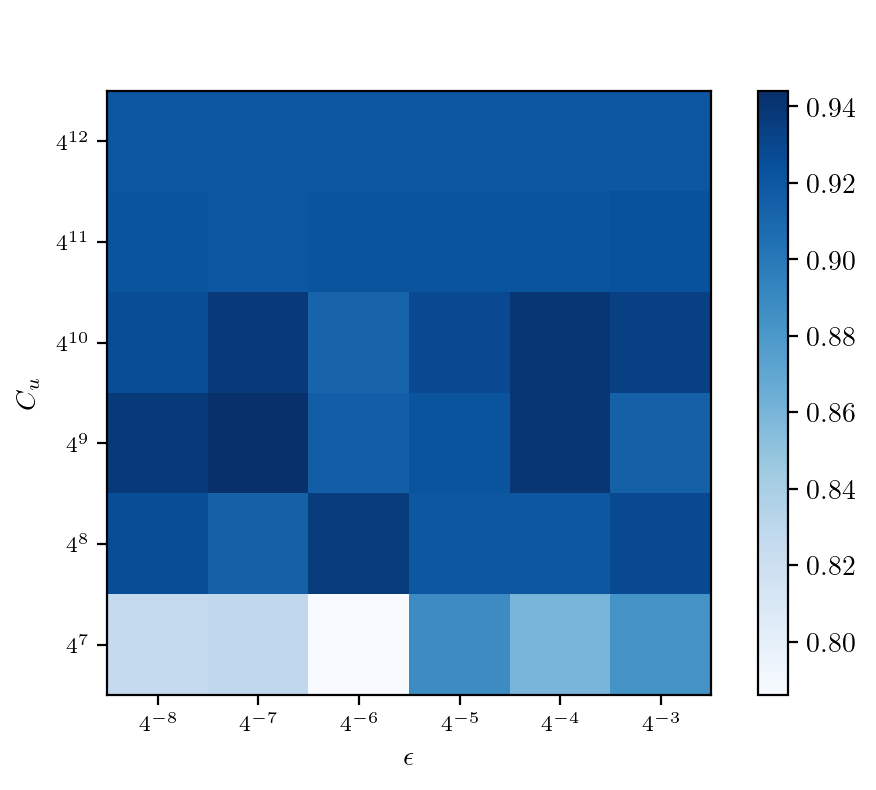}
        \caption{LS-L1-$\mathfrak U$-SQSSVM (Arti-Q2)}
        \label{fig: LS-L1-U-SQS Arti-Q2}
    \end{subfigure}
    
    \begin{subfigure}[H]{0.3\textwidth}
        \includegraphics[width=\textwidth]{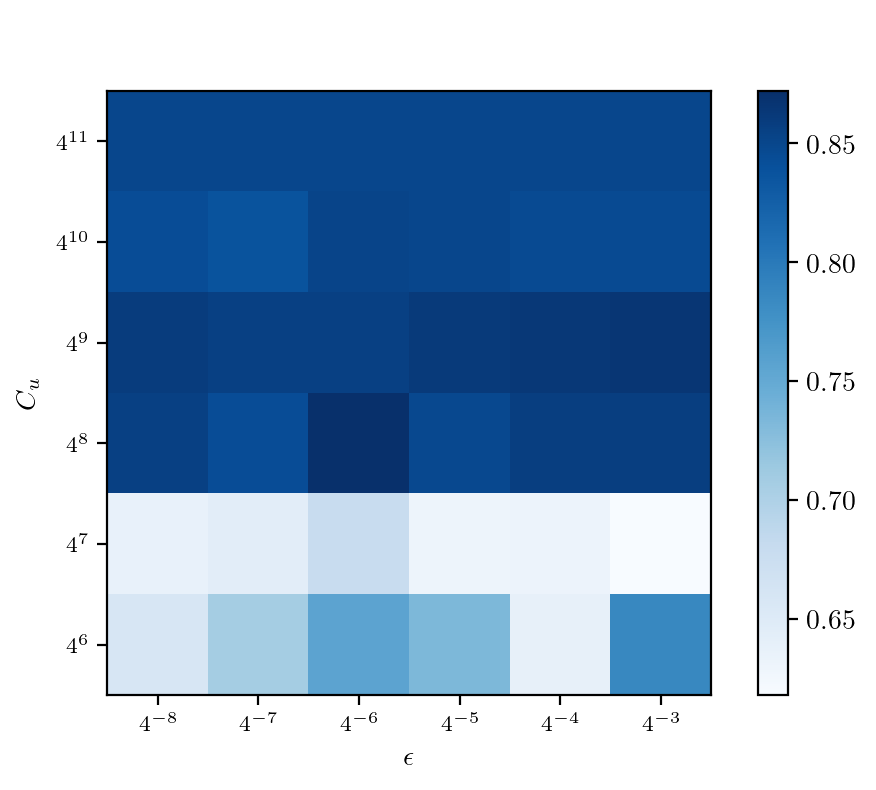}
        \caption{$\mathfrak U$-SQSSVM (ND2by100)}
        \label{fig: U-SQS ND2by100}
    \end{subfigure}
	~
    \begin{subfigure}[H]{0.3\textwidth}
        \includegraphics[width=\textwidth]{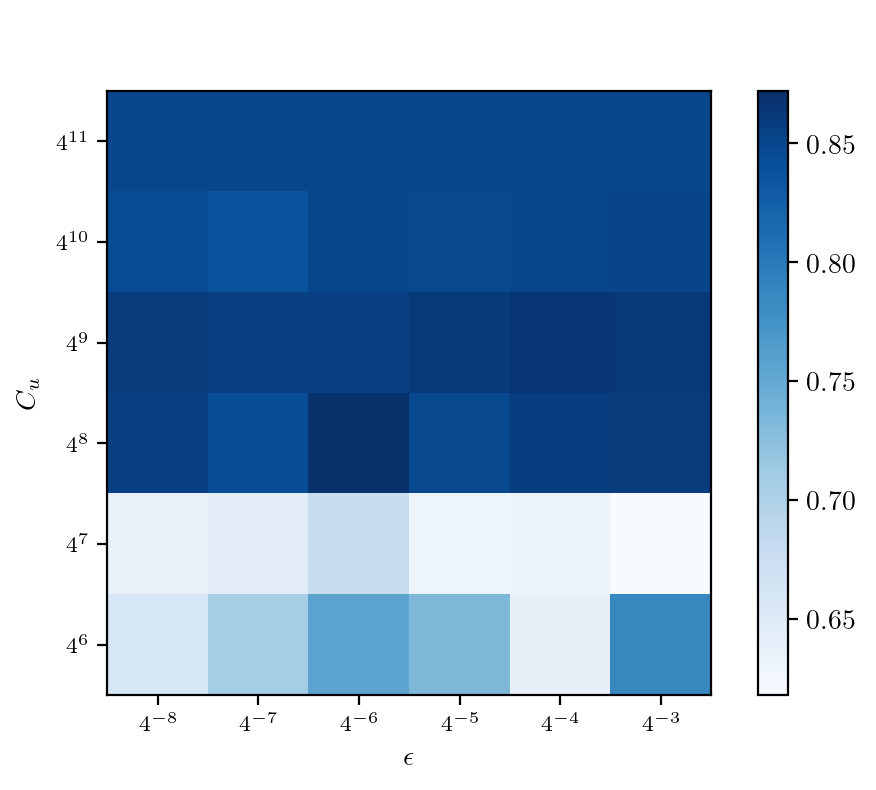}
        \caption{L1-$\mathfrak U$-SQSSVM (ND2by100)}
        \label{fig: L1-U-SQS ND2by100}
    \end{subfigure}
    ~
    \begin{subfigure}[H]{0.3\textwidth}
        \includegraphics[width=\textwidth]{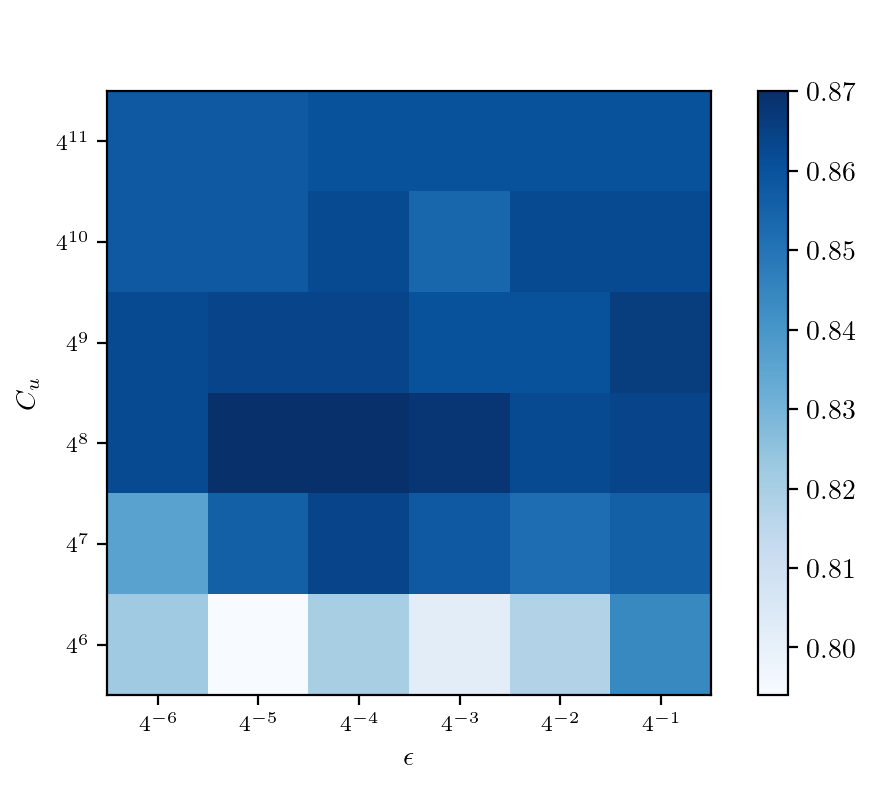}
        \caption{LS-L1-$\mathfrak U$-SQSSVM (ND2by100)}
        \label{fig: LS-L1-U-SQS ND2by100}
    \end{subfigure}    
    
    \caption{Colormaps (accuracy vs. $C_u$ and $\epsilon$).}
    \label{fig: Colormaps acc vs paras}
\end{figure}

\begin{table}[H]
\centering
\begin{tabular}{c|ccccc}\hline
	&		SQSSVM		&		L1-SQSSVM		&		$\mathfrak U$-SQSSVM		&		L1-$\mathfrak U$-SQSSVM		&		LS-L1-$\mathfrak U$-SQSSVM \\\hline
Arti-Q2 & 85.00 & 85.00 & 87.20 & 87.20 & 87.40 \\ 
ND2by200 & 93.00 & 93.00 & 97.00 & 97.00 & 95.00 \\ 
\hline
\end{tabular}
\caption{Highest accuracy scores (\%) on the colormaps in Figure \ref{fig: Colormaps acc vs paras}.}
\label{table: Highest acc on colormaps}
\end{table}

The accuracy scores produced by SQSSSVM, L1-SQSSVM, and all proposed models on these data sets are listed in Table \ref{table: Highest acc on colormaps}. From Figure \ref{fig: Colormaps acc vs paras} and Table \ref{table: Highest acc on colormaps}, it is not hard to see that the accuracy of the proposed models can be improved by adjusting the parameters $C_u$ and $\epsilon$. In addition, notice that the optimal $C_u$ is a relatively big number while the optimal $\epsilon$ is a relatively small number. This helps us reduce the range of grids when using the grid-search method for tuning parameters in practice. It may save much effort when training the models.

Next, we investigate how the amount of the Universum data add to the training set influences the classification accuracy of the proposed models. By fixing all the parameters, the accuracy scores of the proposed models are  plotted against the rate of the Universum data. Notice that the Universum data is randomly created in the data set. Hence, in order to make it statistically meaningful, each curve plotted in Figure \ref{fig: acc vs urate} is the average of ten repeated experiments. 

\begin{figure}[H]
    \centering
    \begin{subfigure}[H]{0.48\textwidth}
        \includegraphics[width=\textwidth]{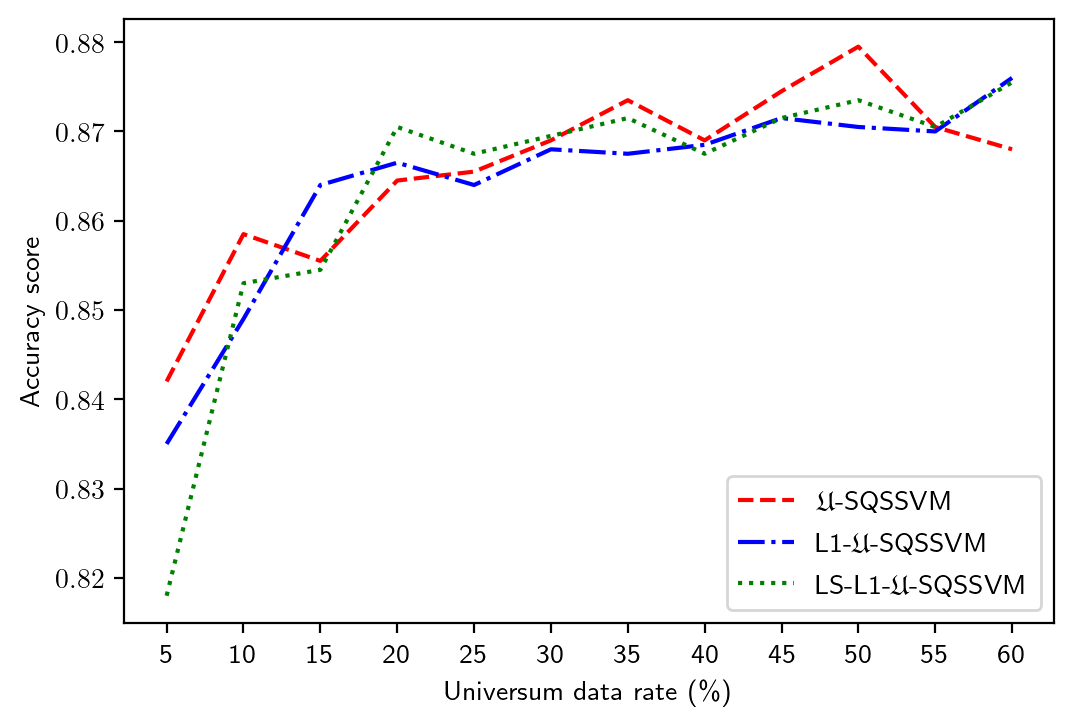}
        \caption{ND200}
        \label{fig: ND200 acc vs urate}
    \end{subfigure}
	~
    \begin{subfigure}[H]{0.48\textwidth}
        \includegraphics[width=\textwidth]{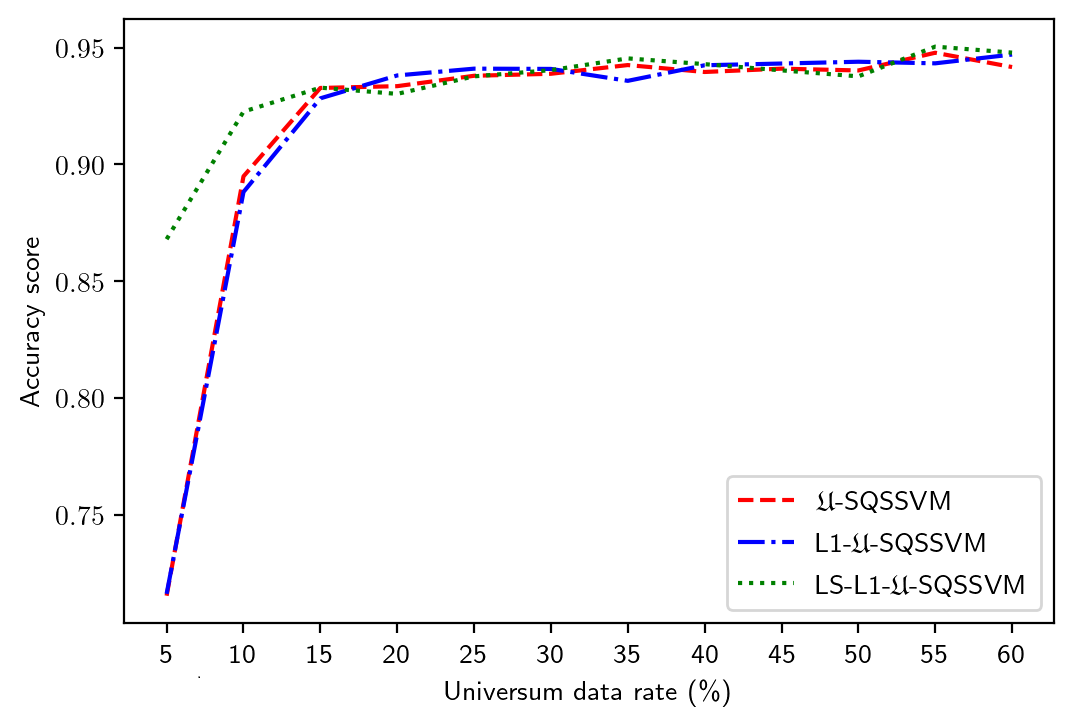}
        \caption{Seeds}
        \label{fig: seeds acc vs urate}
    \end{subfigure}
	~
    \begin{subfigure}[H]{0.50\textwidth}
        \includegraphics[width=\textwidth]{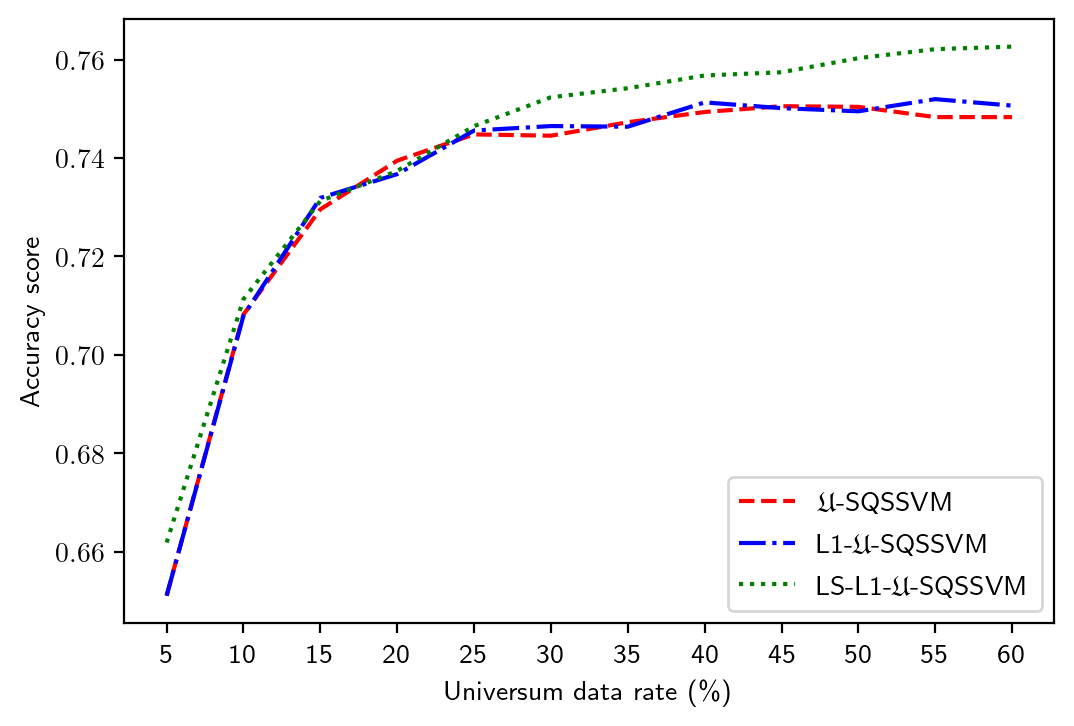}
        \caption{Pima}
        \label{fig: pima acc vs urate}
    \end{subfigure}
    
    \caption{Accuracy scores vs. Universum data rates}
    \label{fig: acc vs urate}
\end{figure}

From the Figure \ref{fig: acc vs urate}, we observe that the proposed models become more accurate as the amount of the Universum data increases. For each model, the accuracy score will approach the highest accuracy score the model can achieve with the fixed parameters. In other words, when there are enough Universum data, increasing its amount will not significantly affect the classification accuracy.

In addition, we also train the proposed models on a linearly separable data set. The classification surfaces are plotted in the following Figure \ref{fig: linearly separable data set}. The training data is plotted as the hollow points and the testing data is plotted as the solid points. The green points are the Universum data generated with the training data.

\begin{figure}[H]
    \centering
        \includegraphics[width=0.7\textwidth]{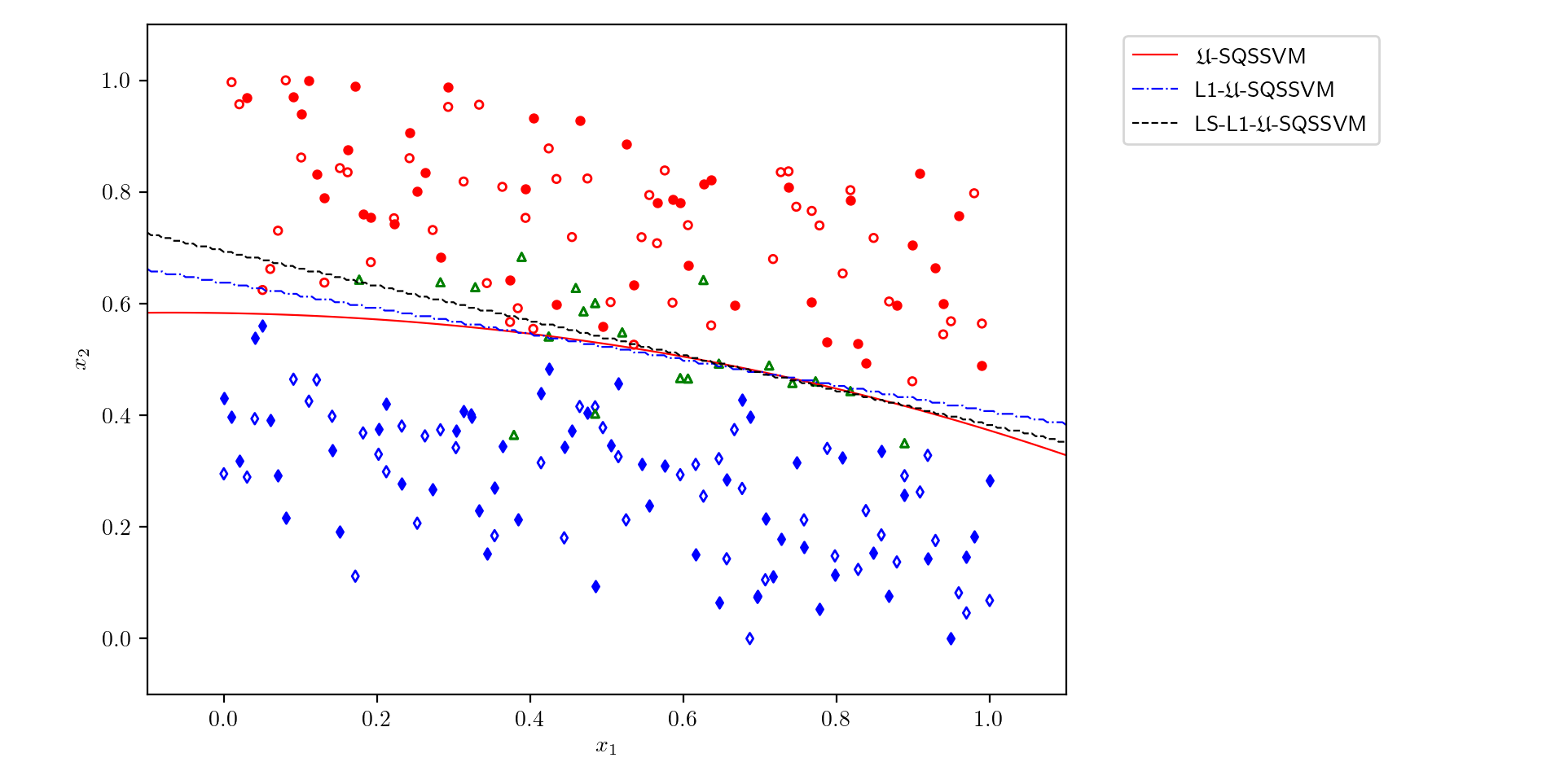}
    \caption{Linearly separable data set ($\lambda = 4096$)}.
    \label{fig: linearly separable data set}
\end{figure}

By fixing $\mu, C_u$ and $\epsilon$ for all three proposed models and letting $\lambda = 4096$ as a big pentalty parameter on the L1 terms in L1-$\mathfrak U$-SQSSVM and LS-L1-$\mathfrak U$-SQSSVM. Although all the three proposed models are capable of classifying this linearly separable data set completely, the $\mathfrak U$-SQSSVM model is not able to generate a linear hyperplane as the other two models do. It indicates that the L1 terms may help adjust the flexibility of the separation surface when applied to real-world applications. Notice that the accuracy highly depends on the parameters, so selecting the parameters is very important for the performance of the proposed models. The grid-search method is a commonly used approach to tune parameters for machine learning models \cite{gao2021kernel,luo2016soft} and we adopted this method to choose the best parameters. The parameters were selected from, for example, $\log_2 \mu \in \{-4, \dots, 20\}$, $\log_2 \lambda \in \{-8, \dots, 20\}$, $\log_2 C_u \in \{-4, \dots, 10\}$, and $\log_2 \epsilon \in \{-8, \dots, 0\}$. In addition, parameter $\gamma$ is selected from $\log_2 \gamma \in \{-4, \dots, 4\}$ for the RBF kernel $k(x_i,x_j)=\exp{\left({-\|{x_i-x_j}\|^2}/{\gamma^2}\right)}$.

\subsection{Experiments on Artificial Data Sets}
\label{subsection: Experiments on artificial data sets}

In this subsection, we conduct some experiments on artificially designed data sets to see how the proposed models perform on data sets with different patterns. Before that, we pre-test the proposed models with two easy artificial data sets. As plotted in Figure \ref{fig: plot solutions with different data patterns.}, one data set is quadratically separable and the other one has normal in-class distributions. The hollow points in the plots are the training data points and the green points are the generated Universum data points. The solid points are for testing. The separation hyperplanes or surfaces produced by the proposed models and the benchmark models are also plotted.

\begin{figure}[htbp]
    \centering
    \begin{subfigure}[b]{0.7\textwidth}
        \includegraphics[width=\textwidth]{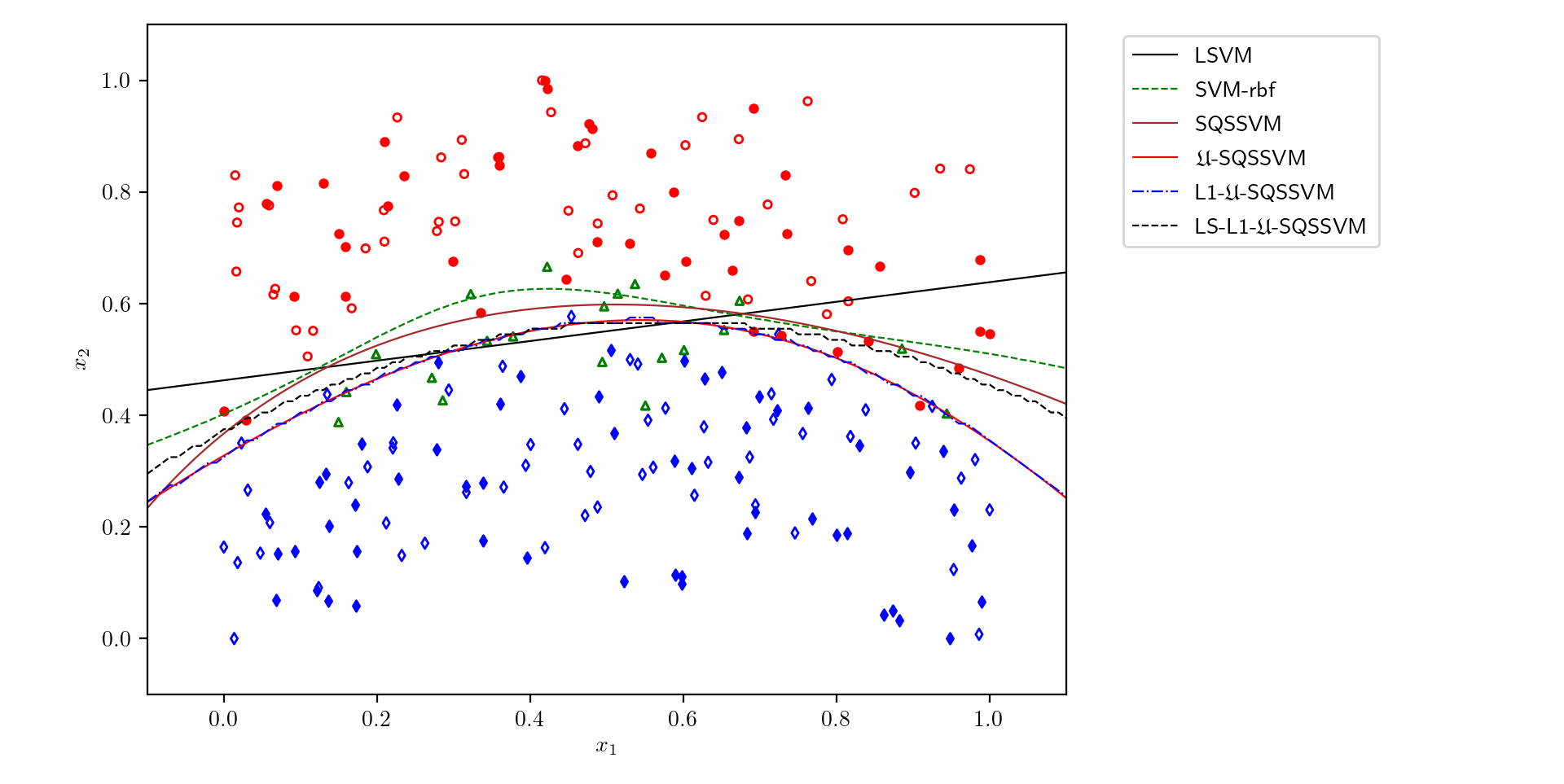}
        \caption{Quadratically separable data set.}
        \label{fig: Q2N200trte_solu}
    \end{subfigure}
	
	\begin{subfigure}[b]{0.7\textwidth}
        \includegraphics[width=\textwidth]{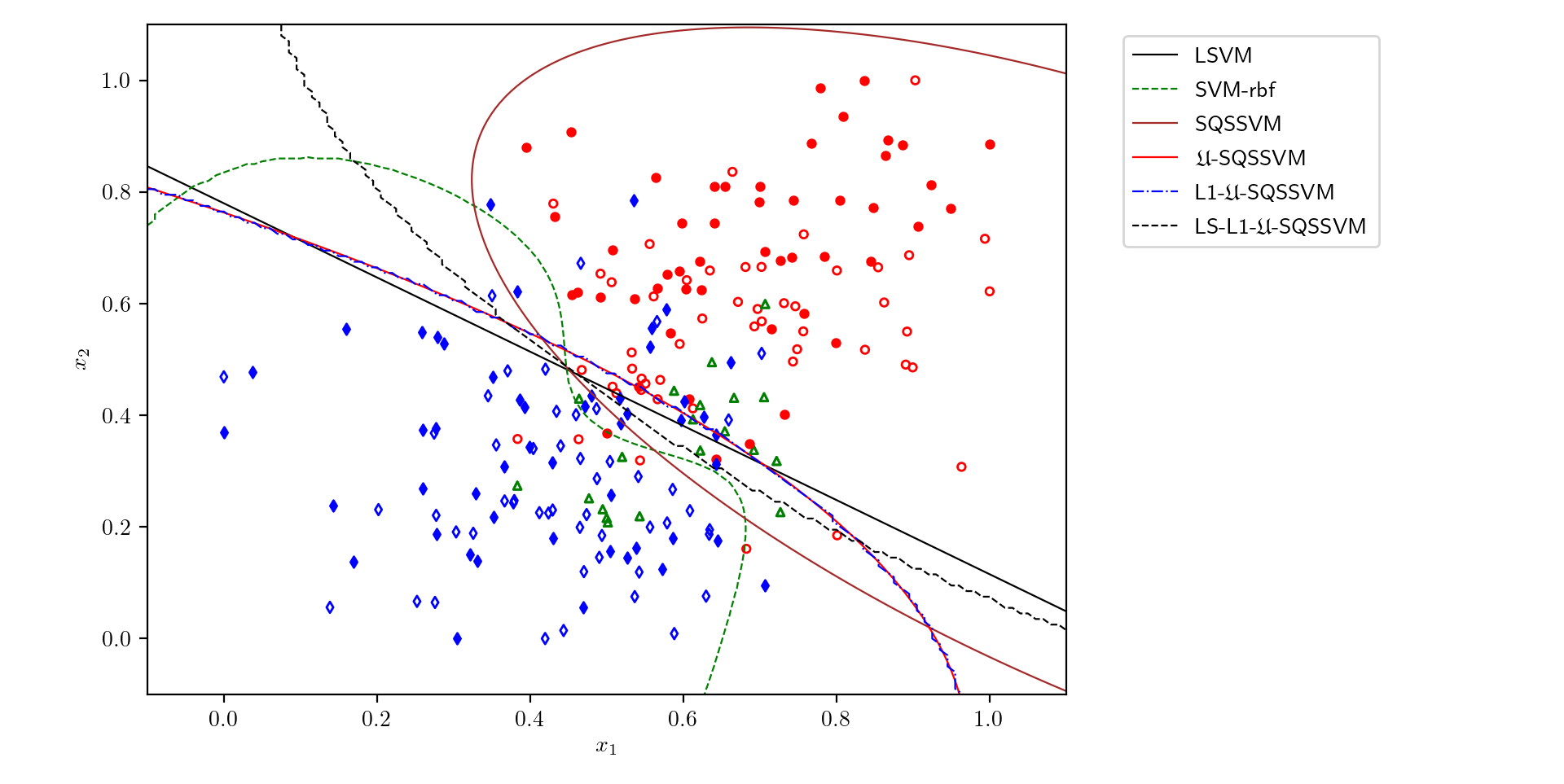}
        \caption{Normally distributed data set.}
        \label{fig: ND200trte_solu}
    \end{subfigure}
    
    \caption{Plot separation surfaces on data sets with different patterns.}
    \label{fig: plot solutions with different data patterns.}
\end{figure}

From the results in Figure \ref{fig: plot solutions with different data patterns.}, we can see the classification results generated by the proposed models are as good as, or even better than those of other well-studied models. 

Next, more experiments are conducted on different artificial data sets with a five-fold cross-validation procedure. 
Some basic information of the artificial data sets is listed in Table \ref{table: artificial data sets info}.

\begin{table}[H]
\centering
\begin{tabular}{ccc}
\hline Data set & \# of data points (Class 1/Class 2) & \# of features \\
\hline
Arti-Q1	&	100/100	&	2	\\
Arti-Q2	&	100/100	&	2	\\
Arti-Q3	&	400/400	&	2	\\ 
npc550	&	191/359	&	32	\\
\hline
\end{tabular}
\caption{Basic information of artificial data sets.}
\label{table: artificial data sets info}
\end{table}

The Arti-Q1 data set is a quadratically separable data set. The Arti-Q2 and Arti-Q3 data sets are data sets with quadratic patterns but non-separable. The npc500 data set has in-class normal distributions with a relatively larger number of features. For each data set, the mean and the standard deviation of the accuracy scores of each tested model are recorded. 

    

\begin{table}[H]
\centering
\begin{tabular}{c|cccc}
\hline
\multirow{2}{*}{Models} & \multicolumn{4}{c}{Accuracy score / standard deviation (\%)} \\
\cline{2-5} &	Arti-Q1			&	Arti-Q2			&	Arti-Q3			&	npc550			\\
\hline
SQSSVM	&	\textbf{100.00}	/	0.00	&	77.75	/	5.46	&	78.75	/	4.38	&	95.73	/	1.77	\\
L1-SQSSVM	&	\textbf{100.00}	/	0.00	&	77.75	/	5.46	&	78.75	/	4.38	&	95.82	/	1.78	\\
LSVM	&	64.25	/	8.58	&	57.50	/	8.66	&	60.94	/	2.67	&	87.27	/	4.18	\\
SVM-rbf	&	99.25	/	1.21	&	61.00	/	8.60	&	80.44	/	3.74	&	91.09	/	4.45	\\
$\mathfrak U$-SVM	&	58.25	/	10.93	&	63.25	/	8.58	&	63.50	/	3.30	&	73.36	/	20.79	\\
$\mathfrak U$-SVM-rbf	&	64.75	/	10.10	&	64.25	/	8.25	&	63.81	/	4.04	&	67.55	/	22.39	\\
$\mathfrak U$-SQSSVM	&	\textbf{100.00}	/	0.00	&	\textbf{81.00}	/	7.92	&	81.00	/	4.12	&	\textbf{95.91}	/	2.47	\\
L1-$\mathfrak U$-SQSSVM	&	\textbf{100.00}	/	0.00	&	\textbf{81.00}	/	7.92	&	81.00	/	4.12	&	95.82	/	2.61	\\
LS-L1-$\mathfrak U$-SQSSVM	&	\textbf{100.00}	/	0.00	&	79.50	/	8.40	&	\textbf{81.44}	/	3.06	&	95.64	/	1.86	\\
\hline
\end{tabular}
\caption{Artificial data results.}
\label{table: Artificial data results results}
\end{table}

From the results in Table \ref{table: Artificial data results results}, we observe that the proposed models perform as well as other models on the Arti-Q1 data set. Indeed, the proposed $\mathfrak U$-SQSSVM, L1-$\mathfrak U$-SQSSVM, and LS-L1-$\mathfrak U$-SQSSVM models all produce quadratic separation surfaces for binary classification, and they are expected to classify the quadratically separable Arti-Q1 data set perfectly. For the other three data sets, the classification accuracy is improved by the proposed models. It motivates us to conduct more computational experiments to see the performance of the proposed models compared with others.

\subsection{Experiments on Public Benchmark Data Sets}
\label{subsection: experiments on public benchmark data sets}

In this subsection, we investigate the performance of the proposed models on some public benchmark data sets. In Table \ref{table: public benchmark data sets info}, some basic information of the utilized benchmark data sets is listed.

    \begin{savenotes}
    \begin{table}[H]
	\centering
	\begin{tabular}{ccc}
	\hline Data set & \# of data points (Class 1/Class 2) & \# of features \\
	\hline
	AUScredit	&	307/383	&	14	\\
	CTG	&	471/1655	&	22	\\
	DEUcredit	&	300/700	&	20	\\
	liver	&	134/185	&	6	\\
	MAGIC	&	6688/12332	&	10	\\
	seeds	&	66/68	&	7	\\
	svmguide1	&	1089/2000	&	4	\\
	svmguide3	&	296/947	&	21	\\
	wine	&	59/71	&	13	\\
	JAPcredit	&	296/375	&	15	\\
	heart	&	120/150	&	13	\\
	ecoli	&	143/193	&	7	\\
	wholesale	&	142/298	&	7	\\
	blood	&	178/570	&	4	\\
	sonar	&	96/110	&	60	\\
	fruit	&	5000/5000	&	5	\\
    \hline
	\end{tabular}
    
	\caption{Basic information of public benchmark data sets.\protect\footnote{The sources of data sets can be found here: \url{https://github.com/tonygaobasketball/Sparse-UQSSVM-Models-for-Binary-Classification}}}
	\label{table: public benchmark data sets info}
	\end{table}
	\end{savenotes}

For each data set, the five-fold cross-validation is adopted in the training process for tuning parameters. The parameters are tuned by using the grid-search method. The mean and the standard deviation of the accuracy scores of each tested model are recorded. We also record the training CPU time of each tested model to compare the computational efficiency.


\begin{sidewaystable} 
\centering
\resizebox{\textwidth}{!}{
\begin{tabular}{c|ccccccccc}\hline
& \multicolumn{9}{c}{Accuracy score / standard deviation (\%)} \\
\cline{2-10}
	&		SQSSVM		&		L1-SQSSVM		&		LSVM		&		SVM-rbf		&		$\mathfrak U$-SVM		&		$\mathfrak U$-SVM-rbf		&		$\mathfrak U$-SQSSVM		&		L1-$\mathfrak U$-SQSSVM		&		LS-L1-$\mathfrak U$-SQSSVM		\\
\hline
	AUScredit	&	85.43	/	2.68	&	85.43	/	2.68	&	85.29	/	2.98	&	85.14	/	2.76	&	83.41	/	6.84	&	85.07	/	3.50	&	85.43	/	2.68	&	85.43	/	2.68	&	\textbf{86.52}	/	2.82	\\
	CTG	&	98.12	/	0.92	&	98.12	/	0.92	&	96.64	/	0.90	&	98.26	/	0.55	&	45.70	/	4.62	&	91.65	/	1.50	&	98.16	/	0.26	&	98.16	/	0.26	&	\textbf{98.35}	/	0.39	\\
	DEUcredit	&	74.90	/	2.92	&	75.20	/	3.05	&	75.95	/	2.50	&	75.45	/	2.73	&	41.15	/	7.05	&	73.20	/	2.72	&	75.30	/	2.06	&	75.50	/	2.66	&	\textbf{76.05}	/	2.57	\\
	liver	&	92.85	/	6.34	&	95.13	/	5.01	&	95.51	/	4.57	&	94.79	/	4.68	&	92.51	/	4.70	&	91.41	/	7.11	&	96.30	/	3.49	&	96.30	/	3.49	&	\textbf{97.78}	/	2.59	\\
	MAGIC	&	\textbf{86.15}	/	0.41	&	\textbf{86.15}	/	0.41	&	79.16	/	0.52	&	85.50	/	0.44	&	57.25	/	11.32	&	65.67	/	9.63	&	\textbf{86.15}	/	0.41	&	\textbf{86.15}	/	0.41	&	85.26	/	0.45	\\
	seeds	&	92.85	/	6.34	&	95.13	/	5.01	&	95.51	/	4.57	&	94.79	/	4.68	&	92.51	/	4.70	&	91.41	/	7.11	&	96.30	/	3.49	&	96.30	/	3.49	&	\textbf{97.78}	/	2.59	\\
	svmguide1	&	\textbf{96.58}	/	1.02	&	\textbf{96.58}	/	1.02	&	93.33	/	0.70	&	95.95	/	0.86	&	79.94	/	10.08	&	87.59	/	18.43	&	\textbf{96.58}	/	1.02	&	\textbf{96.58}	/	1.02	&	94.71	/	1.27	\\
	svmguide3	&	83.11	/	2.90	&	82.94	/	2.44	&	81.90	/	2.25	&	82.38	/	2.99	&	80.73	/	1.55	&	77.11	/	1.74	&	\textbf{84.47}	/	2.35	&	\textbf{84.47}	/	2.35	&	82.94	/	2.46	\\
	wine	&	97.69	/	3.24	&	96.92	/	3.53	&	97.69	/	3.24	&	96.92	/	3.97	&	95.38	/	8.27	&	63.85	/	20.21	&	\textbf{99.62}	/	1.22	&	\textbf{99.62}	/	1.22	&	99.23	/	1.62	\\
	pima	&	76.76	/	2.17	&	\textbf{77.86}	/	2.45	&	\textbf{77.86}	/	2.45	&	76.95	/	2.72	&	73.05	/	3.23	&	70.37	/	6.48	&	76.76	/	2.17	&	\textbf{77.86}	/	2.45	&	77.08	/	2.48	\\
	JAPcredit	&	86.59	/	2.57	&	86.36	/	2.69	&	86.36	/	2.67	&	85.37	/	3.00	&	84.13	/	6.06	&	86.51	/	2.81	&	86.59	/	2.57	&	86.36	/	2.69	&	\textbf{86.82}	/	2.66	\\
	heart	&	81.30	/	4.14	&	82.59	/	3.40	&	83.33	/	4.94	&	82.22	/	4.29	&	83.33	/	4.70	&	68.33	/	11.33	&	81.48	/	4.36	&	82.59	/	3.40	&	\textbf{83.52}	/	3.95	\\
	ecoli	&	96.43	/	2.74	&	96.43	/	2.74	&	96.13	/	1.89	&	96.28	/	1.90	&	94.94	/	2.74	&	94.94	/	2.74	&	96.43	/	2.74	&	96.43	/	2.74	&	\textbf{97.02}	/	1.99	\\
	wholesale	&	90.80	/	2.42	&	91.36	/	2.41	&	90.91	/	2.68	&	90.34	/	3.83	&	89.32	/	2.99	&	89.32	/	2.99	&	91.02	/	1.96	&	\textbf{91.48}	/	2.35	&	90.00	/	2.38	\\
	blood	&	76.47	/	2.27	&	76.40	/	2.21	&	76.20	/	2.41	&	76.80	/	2.09	&	30.00	/	13.31	&	72.47	/	8.55	&	76.67	/	4.11	&	77.94	/	3.25	&	\textbf{78.00}	/	3.49	\\
	sonar	&	87.93	/	3.60	&	88.17	/	4.02	&	76.35	/	6.23	&	85.55	/	5.48	&	61.86	/	9.67	&	55.09	/	6.99	&	86.49	/	4.09	&	\textbf{88.93}	/	4.59	&	\textbf{88.93}	/	5.27	\\
	fruit	&	92.78	/	0.65	&	92.72	/	0.66	&	93.56	/	0.49	&	93.01	/	0.64	&	50.01	/	1.00	&	50.01	/	1.00	&	93.05	/	0.70	&	92.91	/	0.68	&	\textbf{93.87}	/	0.44	\\
\hline
\end{tabular}
}
\caption{Public benchmark data accuracy results.}
\label{table: benchmark data accuracy results}
\end{sidewaystable} 


\begin{sidewaystable} 
\centering
\begin{tabular}{c|ccccccccccc}
\hline	& \multicolumn{9}{c}{CPU time (s)} \\
\cline{2-10} &		SQSSVM		&		L1-SQSSVM		&		LSVM		&		SVM-rbf		&		$\mathfrak U$-SVM		&		$\mathfrak U$-SVM-rbf		&		$\mathfrak U$-SQSSVM		&		L1-$\mathfrak U$-SQSSVM		&		LS-L1-$\mathfrak U$-SQSSVM		\\
\hline
	Arti-Q1	&	0.031	&	0.033	&	0.022	&	0.008	&	0.180	&	3.448	&	0.045	&	0.045	&	$\bm{<0.001}$	\\
	Arti-Q2	&	0.025	&	0.028	&	0.002	&	0.001	&	0.047	&	0.062	&	0.021	&	0.029	&	$\bm{<0.001}$	\\
	Q2by800	&	0.031	&	0.033	&	0.022	&	0.008	&	0.180	&	3.448	&	0.045	&	0.045	&	$\bm{<0.001}$	\\
	npc550	&	1.696	&	1.541	&	0.115	&	\textbf{0.004}	&	0.128	&	0.996	&	1.991	&	1.822	&	0.192	\\
	AUScredit	&	0.160	&	0.168	&	0.174	&	0.011	&	0.145	&	1.464	&	0.227	&	0.233	&	\textbf{0.005}	\\
	CTG	&	1.151	&	1.168	&	0.213	&	\textbf{0.014}	&	2.575	&	148.545	&	1.617	&	1.607	&	0.084	\\
	DEUcredit	&	0.448	&	0.462	&	0.763	&	0.025	&	0.340	&	3.429	&	0.604	&	0.606	&	\textbf{0.010}	\\
	liver	&	0.039	&	0.044	&	0.007	&	0.014	&	0.052	&	1.595	&	0.049	&	0.051	&	\textbf{0.002}	\\
	MAGIC	&	2.907	&	3.105	&	10.411	&	6.446	&	1013.818	&	1407.144	&	9.418	&	9.274	&	\textbf{0.014}	\\
	seeds	&	0.034	&	0.038	&	0.001	&	$\bm{<0.001}$	&	0.039	&	0.068	&	0.043	&	0.040	&	0.013	\\
	svmguide1	&	0.132	&	0.137	&	0.066	&	0.042	&	5.968	&	24.764	&	0.270	&	0.274	&	$\bm{<0.001}$	\\
	svmguide3	&	0.756	&	0.778	&	0.485	&	0.120	&	0.749	&	23.508	&	1.004	&	1.062	&	\textbf{0.081}	\\
	wine	&	0.077	&	0.077	&	$\bm{<0.001}$	&	0.001	&	0.037	&	0.071	&	0.085	&	0.089	&	0.043	\\
	pima	&	1.830	&	1.657	&	\textbf{0.135}	&	0.006	&	0.173	&	1.083	&	2.209	&	1.945	&	0.271	\\
	JAPcredit	&	1.745	&	1.577	&	0.122	&	\textbf{0.006}	&	0.130	&	1.029	&	2.053	&	1.851	&	0.217	\\
	heart	&	0.077	&	0.087	&	\textbf{0.012}	&	0.003	&	0.049	&	0.225	&	0.099	&	0.165	&	0.026	\\
	ecoli	&	0.035	&	0.043	&	$\bm{<0.001}$	&	0.008	&	0.050	&	0.062	&	0.046	&	0.044	&	0.005	\\
	wholesale	&	0.059	&	0.054	&	\textbf{0.003}	&	\textbf{0.003}	&	0.069	&	0.082	&	0.066	&	0.062	&	0.016	\\
	blood	&	0.045	&	0.045	&	0.012	&	0.033	&	0.131	&	3.123	&	0.056	&	0.060	&	$\bm{<0.001}$	\\
	sonar	&	20.457	&	12.824	&	0.017	&	$\bm{<0.001}$	&	0.068	&	0.153	&	19.902	&	12.949	&	44.527	\\
	fruit	&	0.497	&	0.506	&	0.316	&	0.427	&	168.744	&	286.431	&	1.928	&	1.945	&	\textbf{0.010}	\\
\hline
\end{tabular}
\caption{Training CPU time.}
\label{table: CPU time results}
\end{sidewaystable} 

To summarize the results in Tables \ref{table: benchmark data accuracy results} and \ref{table: CPU time results}, we plot the histograms in Figure \ref{fig: histograms public benchmark acc and CPUt} to compare the average of mean accuracy scores and the median of CPU time of all tested models on all the benchmark data sets.

\begin{figure}[htbp]
    \centering
    \begin{subfigure}[b]{\textwidth}
        \includegraphics[width=\textwidth]{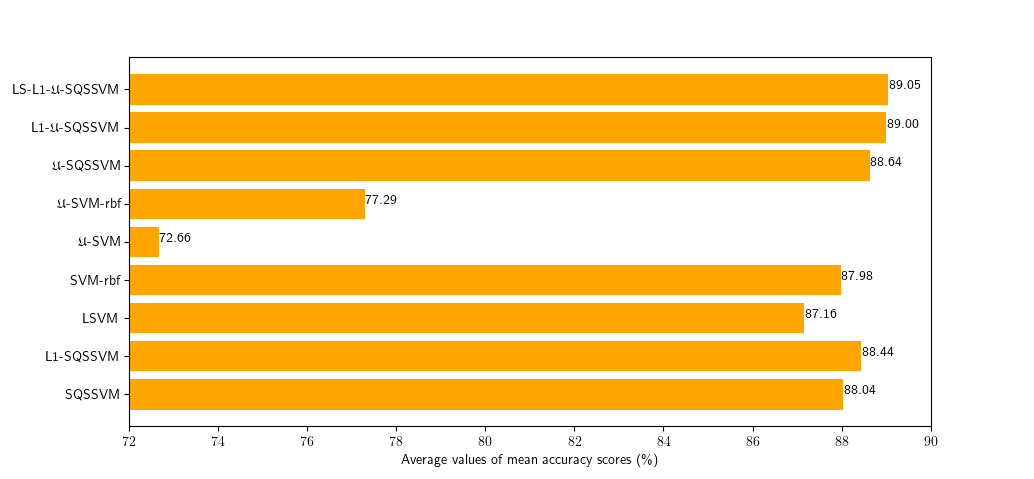}
        \caption{Average of mean accuracy scores on public benchmark data sets.}
        \label{fig: Barh-MeanAcc}
    \end{subfigure}
	
    \begin{subfigure}[b]{\textwidth}
        \includegraphics[width=\textwidth]{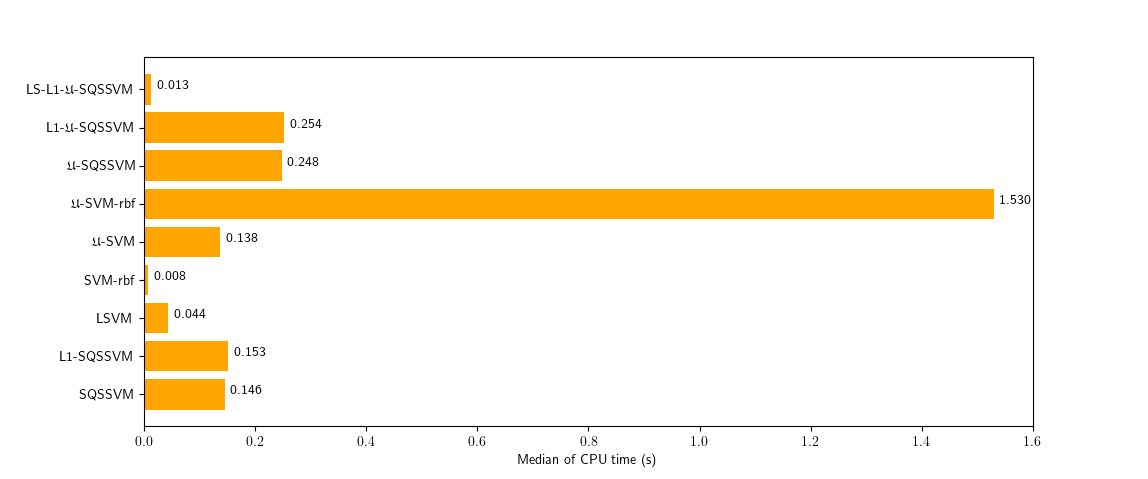}
        \caption{Median of CPU time on public benchmark data sets.}
        \label{fig: Barh-MedianCPUt}
    \end{subfigure}
    \caption{}
    \label{fig: histograms public benchmark acc and CPUt}
\end{figure}

From the results in Table \ref{table: benchmark data accuracy results}-\ref{table: CPU time results}, we have the following observations:
\begin{itemize}
    \item The proposed $\mathfrak U$-SQSSVM, L1-$\mathfrak U$-SQSSVM and LS-L1-$\mathfrak U$-SQSSVM models produce the highest or the second highest mean accuracy scores on the tested data sets. Notice that there are both small and large scale data sets listed in Table \ref{table: public benchmark data sets info}, which approves the effectiveness of the proposed models on classifying data sets in different scales.
    \item The CPU time of the proposed $\mathfrak U$-SQSSVM and L1-$\mathfrak U$-SQSSVM depends on the number of data points in the data sets. Even though for large data sets, like the MAGIC or the fruit data sets, the CPU time is high, it is still acceptable. Moreover, the CPU time of the proposed LS-L1-$\mathfrak U$-SQSSVM model is short on most of the data sets. Although the algorithm might not be efficient enough when the number of features in the data set is too large, the general efficiency is acceptable. In all, the CPU time listed in Table \ref{table: CPU time results} indicates the efficient performance of the proposed models on real-world applications.
    \item For each data set, the Universum models, i.e., $\mathfrak U$-SQSSVM and L1-$\mathfrak U$-SQSSVM are at least as accurate as the corresponding SQSSVM model or the L1-SQSSVM model. In other words, they take advantage of both kernel-free QSSVM models and the Universum data. Compared with $\mathfrak U$-SVM and $\mathfrak U$-SVM-rbf models, the proposed Universum models produce smaller standard deviations, which indicates better stability. 
    
    \item From Figure \ref{fig: Barh-MeanAcc}, we notice that all the proposed models have higher average values of mean accuracy scores than those of other models. In addition, the LS-L1-$\mathfrak U$-SQSSVM model implemented with the proposed algorithm is as efficient as the SVM-rbf implemented with the well-developed Sklearn package. In all, it indicates that our proposed models, especially the LS-L1-$\mathfrak U$-SQSSVM model, have a better general performance compared with other tested models.
    
\end{itemize}

\section{Conclusions}
\label{sec: Conclusions}

In this paper, we have proposed three kernel-free Universum quadratic surface support vector machine models for binary classification. Certain theoretical properties have been rigorously investigated and an efficient algorithm has been proposed to implement the proposed LS-L1-$\mathfrak U$-SQSSVM model. Computational experiments have been conducted to verify the effectiveness and the computational efficiency of the proposed models. Some major findings are summarized as follows.

\begin{itemize}
    \item By incorporating Universum data into the kernel-free QSSVM models, we  proposed $\mathfrak U$-QSSVM models. To adjust the flexibility of the separation surface and to obtain a possible sparsity pattern, we have utilized an $\ell_1$ norm regularization to obtain L1-$\mathfrak U-$QSSVM models. 
    
    \item In addition, the least squares version of the L1-$\mathfrak U$-SQSSVM model has been proposed, which is denoted as LS-L1-$\mathfrak U$-SQSSVM. Instead of using a standard  numerical optimization solver, we have designed an efficient algorithm for implementing the proposed LS-L1-$\mathfrak U$-SQSSVM. 
    
    \item Some theoretical properties, including the solution existence, $z$-uniqueness, vanishing margin property and the bounds for $c$, have been studied for $\mathfrak U$-QSSVM and L1-$\mathfrak U$-QSSVM models. Moreover, we have investigated the conditions on nonzero optimal solution of the LS-L1-$\mathfrak U$-QSSVM models.
    
    \item Numerical experiments have been conducted to show the influence of the parameters on classification accuracy. In addition, the promising numerical results on some artificial and some public benchmark data sets imply the effectiveness of the proposed models in solving real-world binary classification problems. Last but not the least, the short training CPU time verifies the high efficiency of the proposed algorithm for the LS-L1-$\mathfrak U$-SQSSVM model.
    
\end{itemize}

This research might be extended to some additional research works. An immediate future work is to investigate the robustness of the proposed models for binary classification \cite{wang2018robust}. Moreover, it would be interesting to investigate how the proposed models perform on imbalanced data, such as the credit scoring data. Generalizing these models for multi-class classification also needs a through future study.


\section*{Acknowledgments} 
The work of H. Moosaei was  supported by the Czech Science Foundation Grant P403-18-04735S and Center for Foundations of Modern Computer Science
(Charles Univ. project UNCE/SCI/004). 
 The work of M.  Hlad\'{i}k was  supported by the Czech Science Foundation Grant P403-18-04735S.

\section*{Conflict of interest}
 The authors declare that they have no conflicts of interest.

\bibliographystyle{plain}
\bibliography{main}

\medskip
\medskip

\end{document}